\documentclass[a4paper,12pt]{article}
\usepackage[utf8]{inputenc}
\usepackage{url}
\usepackage{multirow}
\usepackage{amsthm}
\usepackage{mathtools}
\usepackage{xcolor}
\usepackage{amsmath,graphicx,latexsym,amssymb}
\newcommand{\mR}{\mathbb{R}}

\renewcommand{\leq}{\leqslant}
\renewcommand{\geq}{\geqslant}
\newtheorem{prop}{Proposition}
\newtheorem{Def}{Definition}
\newtheorem{lemma}{Lemma}
\newtheorem{claim}{Claim}

\usepackage{hyperref}
\usepackage{comment}
\usepackage{xcolor}
\usepackage{amsmath,graphicx,latexsym,amssymb,amsthm}

\newcommand{\mc}{\mathcal}

\newcommand{\mE}{\mathbb{E}}

\newcommand{\om}{\omega}

\newcommand{\tha}{\theta}
\newcommand{\lam}{\lambda}

\usepackage{tabularx}
\usepackage{booktabs}

\renewcommand{\leq}{\leqslant}
\renewcommand{\geq}{\geqslant}

\newtheorem{assume}{Assumption}

\usepackage{cite}
\usepackage{algpseudocode}
\usepackage{algorithmicx}
\usepackage{algorithm}
\usepackage{caption}
\usepackage{eufrak}
\usepackage{diagbox}
\usepackage{adjustbox}
\usepackage{subcaption}

\DeclareMathOperator{\eps}{\epsilon}

\DeclareMathOperator{\supp}{supp}

\newcommand{\blue}[1]{\textcolor{blue}{#1}}

\usepackage{graphicx}

\usepackage{algorithmicx}
\usepackage{algorithm}
\usepackage{algpseudocode}

\title{Constrained Density Estimation using Optimal Transport}
\author{Yinan Hu and Esteban G. Tabak}
\begin{document}
\maketitle
\begin{abstract}
A novel framework for density estimation under expectation constraints is proposed. The framework minimizes the Wasserstein distance between the estimated density and a prior, subject to the constraints that the expected value of a set of functions adopts or exceeds given values. The framework is generalized to include regularization inequalities to mitigate the artifacts in the target measure. An annealing-like algorithm is developed to address non-smooth constraints, with its effectiveness demonstrated through both synthetic and proof-of-concept real world examples in finance. 
\end{abstract}

\section{Introduction}

 The classical optimal transport (OT) problem seeks the map that moves mass from a source to a target measure while minimizing a prescribed cost function. The objective can be formalized in either Monge's \cite{monge1781OT} or Kantronich's formulation \cite{kantorovich1942translocation}, a convex relaxation of the former that considers transport plans instead of deterministic maps. These foundational formulations have wide-ranging applications, including to economics \cite{galichon2016OT_economics} and machine learning \cite{villani2008optimal}. 

In many practical scenarios, the source measure is known or readily inferrable from empirical data but the target measure is not explicitly specified. Instead, it is only constrained by practical requirements or expert knowledge. For example, when applying Monge's formulation to transportation problems, the placement of the mass in the target region may be constrained to lie entirely beyond a certain boundary or within a particular region, rather than by the specification of a precise location for each fraction of the total mass. Similarly, in economic applications, supply and demand may be subject to constraints such as maximal amounts available or minimal amounts required, rather than dictated through precise marginal distributions. 

    Many of these constraints can be naturally formulated in terms of the expected value of specific functions. For example, in nutritional planning, the focus may be on the total daily protein intake rather than the specific food items consumed. In finance, the price of options placed on an asset must agree with their expected value under the asset's underlying risk-free measure. Also, when the target distribution is constrained through observations, 
   these are often presented as expected values arising from repeated measurements.

 
    These considerations lead to developing a generalized optimal transport model which incorporates expectation-based constraints, enhancing and expanding the applicability of optimal transport in fields such as finance, statistical inference and source allocation under uncertainty.  

    This paper introduces and develops such framework. Starting from a source or prior distribution, we seek the target density that minimizes the Wasserstein distance to the prior while satisfying a set of specified expectation constraints.

\subsection{Related Work}
The most common constraints on an estimated density investigated are linear combinations of moments of the desired estimators. Examples include \cite{jones1991variance_inflation}, focusing on the role of variance inflation on the data upon the restored estimator,  and \cite{hall1999density_estimation_constraint}, which proposes a framework involving smoothed bootstrapping constraints. The framed proposed here addresses general equality constraints and adopts the Wasserstein distance as metric to measure the distance between the prior and estimated distributions.

In finance, the pricing of options has long been a focus of research. A generalized path-integral approach is proposed in \cite{bormetti2006pricing_exotic_option_path_integral} to price common exotic options formulated using the classical Black-Scholes model \cite{black_schole1973pricing}, \cite{de2018machine_learning_pricing_options_hedging} uses Gaussian process regression techniques for the learning process to accelerate the pricing of options and\cite{kim2022pricing_option_flow_generative} applies the flow-based generative network RealNVP, which constructs a normalizing flow of pricing distributions, to price path-dependent exotic options. Closer to our work is \cite{avellaneda1997calibrating_KL_minimization}, which modifies the maximum likelihood estimator by introducing expectation constraints as penalty functions and applies it to the pricing of several exotic options. The method proposed in this paper modifies the objective function used in \cite{avellaneda1997calibrating_KL_minimization}, providing an option pricing method with geometric underlying different from maximal likelihood.

\section{Problem formulation}
\label{sec:formulation}

This section formulates the problem of constrained optimal transport.
For convenience, we summarize our notation in table \ref{tab:notation}.
\begin{table}[]
    \centering
    \begin{tabular}{|c|c|}
    \hline
    Notation & Meaning \\
    \hline
      $\rho$  & Prior measure  \\
       $\mu$  & Estimated target measure \\
       $p_{\rho}$ & Probability density function of $\rho$ \\
       $p_{\mu}$ & Probability density function of $\mu$  \\
       $x_1,x_2,\dots$ & Available samples from $\rho$ \\
       $y_1,y_2,\dots$ & Estimated samples from $\mu$ \\
       $\{f_k\},\ k \in \{1,2,\dots, K\}$ & Functions with constrained expected values \\
       \hline
    \end{tabular}
    \caption{Notation used and their meaning}
    \label{tab:notation}
\end{table}
Given a prior measure $\rho$ in the set $\mc{P}$ of Borel measures on a Polish space $X$, we seek the target measure $\mu\in \mc{P}(X)$ with minimal Wasserstein distance to $\rho$ satisfying the expectation constraints 
\begin{equation}
    \int_{X}{f_k(y)d\mu} = \bar{f}_k,
    \label{eq:expectation_constraint}
\end{equation}
where $f_k\in L^1(\mu),\;k=1,2,\dots,K$ are integrable functions and the $\{\bar{f}_k\}$ are given constants.
When $\rho$ and $\mu$ have corresponding probability density functions, we denote them by $p_\rho$ and $p_\mu$.

The corresponding constrained optimal transport can be formulated in two ways, inspired by Monge and Kantorovich's formulations of the classical problem respectively.

\begin{Def}[Constrained optimal transport problem, Monge's formulation] Given a cost function $c:X\times X \rightarrow \mR$, a prior measure $\rho$ and a set of functions $\{f_k\}$ and real values $\{\bar{f}_k\}$, we define the constrained optimal transport problem as:
\begin{equation}
\begin{aligned}   \underset{\substack{T:X\rightarrow X \\ \mu\in \mc{P}(X)}}{\min}\;&  J_c[T] =\int_{X}{c(x,T(x))\ d\rho(x)}, \\
    \text{s.t.}&\; T\#\rho = \mu, \\
    &\int_{X}{f_k(y)\ d\mu}= \bar{f}_k, \ k=1,2,\dots,K.
\end{aligned}
\label{opt:constrained_ot_monge}
   \end{equation}    
\end{Def}

\begin{Def}[Transportation cost, Kantorovich's formulation]
   Given a cost function $c(x,y)$, a prior distribution $\rho$ and a target distribution $\mu$, the Kantorovich transportion cost between them is defined as follows:
   \begin{equation}
       {I}_c(\rho,\mu) = \int_{X\times Y}{c(x,y)\ d\pi(x,y)},
   \end{equation}
    where 
    \begin{equation}
        \pi \in \Pi(\rho,\mu) = \{\pi\in \mc{P}(X\times Y),\;\pi_x = \mu,\;\pi_y = \rho\}.
    \end{equation}
\end{Def}

\begin{Def}[Constrained optimal transport, Kantorovich's formulation]

The Kantorovich formulation of constrained optimal transport is defined as follows:
\begin{equation}
    \begin{aligned}
    \underset{\mu\in \mc{P}(X)}{\min}&\;I_c(\rho,\mu)\\
    \text{s.t.}\;& \int_{X}{f_k(y)\ d\pi_y(y)} = \bar{f}_k,\;k=1,2,\dots,K.
    \end{aligned}
\end{equation}
\end{Def}

Notice that Monge and Kantorovich formulations of constrained optimal transport problems differ from their classical optimal transport counterparts, in that the target measure itself is unknown, not only the transportation map or plan.

\section{Some simple examples}
In order to illustrate the problem and the nature of its solutions, we consider some simple examples for which we can write exact or semi-exact solutions, using as cost the canonical squared distance: $c(x,y) = \|x-y\|_2^2$.
We start with one-dimensional examples, where $X=\mR$.
In regular one-dimensional optimal transport problems with quadratic cost, the optimal map is always monotone (see for instance chapter 2 of \cite{santambrogio2015optimal}.) The following lemma shows that such monotonicity carries over to our formulation.

 \begin{lemma}[The optimal map is monotone]
 \label{lemma:monotone_map}
      If the optimal transport map $T^*:X\rightarrow X$ for the optimization problem \eqref{opt:constrained_ot_monge} exists, it increases monotonically. 
 \end{lemma}
 \begin{proof}
 The optimal solution to the problem consists of the target measure $\mu$ and the map $T$ pushing the prior $\rho$ to $\mu$. The map $T$ necessarily solves the classical optimal transport problem between $\rho$ and $\mu$, since the constraints involve only $\mu$, so any other map yielding the same $\mu$ will automatically satisfy them. Then $T$ inherits all the properties of regular OT, including its monotonicity.
 \end{proof}

We compare the problem's solution for each choice of $\{f_k\}$ to the solution from an alternative formulation of density estimation with constraints, based not on optimal transport but on the Kullback–Leibler divergence between distributions. This was the choice adopted in \cite{avellaneda1997calibrating_KL_minimization} as a metric for the discrepancy between prior and pricing distributions.
They estimate the target density by solving the following optimization problem:

\begin{equation}
\begin{aligned}     \underset{\mu \in \mc{P}(X)}{\min}\;&\text{KL}(\rho\|\mu), \\
    \text{s.t.} &\int_{X}{f_k(y)\ d\mu} = \bar{f}_k, \;k=1,2,\dots,K.
\label{opt:KL_multiple_constraint}
\end{aligned}
\end{equation}
The Kullback–Leibler divergence between distributions is a ``vertical'' measure of dissimilarity, comparing the value of both distributions at each point, while the optimal transport cost provides a ``horizontal'' measure, which quantifies the displacement in $X$-space required to make both distributions identical.

\subsection{Indicator function}
We first consider for $f$ the indicator function for the complement of a closed interval:
\begin{equation}
    f(s) =  \mathbf{1}_{\mR \backslash [a,b]}(s) = \begin{cases}
        1 & s\in \mR \backslash [a,b],\\
        0 & \mbox{o/w},
    \end{cases} \quad \hbox{with constraint } \ \bar{f}=0,
    \label{func:support}
\end{equation}
which requires the support of the target distribution to lie within the interval $[a,b]$.

The solutions of the corresponding OT and KL-based problems are quite different. For KL, we have
\begin{equation}
    p_\mu^{KL}(y) = \begin{cases}
        \frac{p_\rho(y)}{\int_{a}^b{p_\rho(y)dy}} & y\in [a,b], \\
        0 & \mbox{o/w},
    \end{cases}
\end{equation}
where every point within the support of the target is stretched vertically by the same amount, while for OT, 

    \begin{equation}
        \mu^{OT} = \rho|_{[a,b]} + c_a\delta_a + c_b\delta_b,
        \label{opt:sol}
    \end{equation}
    where $\rho|_{[a,b]}$ refers to the restriction of the measure $\rho$ within the interval $[a,b]$, that is, for all $A\subset \mR$,
    \begin{equation}
        \rho|_{[a,b]}(A) = \rho(A\cap [a,b]),
    \end{equation}
    and 
    \begin{equation}
        c_a = \int_{-\infty}^a{p_\rho(z)dz},\;\;c_b = \int_{b}^{\infty}{p_\rho(z)dz},
    \end{equation}
a result proved in appendix \ref{appd:prop_indicator_function}.

\subsection{The RELU function} Another example has as constraint the `RELU' function:
  \begin{equation}
    f(x) = \begin{cases}
    x - \om & x\geq \om, \\
    0 & \mbox{otherwise},
    \end{cases}
    \label{func:relu_1}
\end{equation} 
used for the pricing of options with strike price $\om$.

Here the solution of the KL-based problem is
\begin{equation}
    p^{KL}_\mu(y) = 
    \begin{cases}
        \frac{1}{Z}p_\rho(y)e^{\gamma(y-\om)} & y>\om \\
         \frac{1}{Z}p_\rho(y)& \mbox{o/w}
    \end{cases},
    \label{sol:mu_kl_divergence_relu}
\end{equation}
where
\begin{equation}
    Z = \int_{\om}^{\infty}{p_\rho(y)e^{\gamma(y-\om)}dy} + \int_{-\infty}^{\om}{p_\rho(y)dy}
    \label{eq:normalization}
\end{equation}
and $\gamma$ is the unique constant satisfying
\begin{equation}
    \frac{1}{Z}\int_{\om}^{\infty}{p_\rho(y)e^{\gamma(y-\om)}(y-\om)dy} = \bar{f}.
    \label{sol:lambda}
\end{equation}
By contrast, the solution to the corresponding OT problem involves a horizontal, piecewise constant shifting of elements of the prior measure: for all $A\subseteq \mR$, 
$$
         \mu^{OT}(A) = \rho(T^{-1}_*(A)),
$$
  where $T_*:X\rightarrow Y,\;T_*(x) = x + \lambda \mathbf{1}_{[ x_*,\infty)}(x)$.
The parameters $\lam,x_*$ are determined by the choice among the following two candidate systems that yields the lower value of the objective function.

\smallskip
     1) $$
    \begin{cases}
      x_* =\om - \frac{\lam}{2}, \\
      \int_{x_*}^{\infty}{(x-\om+\lambda)p_\rho(x)dx} = \bar{f},
    \end{cases}
$$

2)

$$
 \begin{cases}
x_*=\om-\lambda, \\
      \int_{\om-\lam}^{\infty}{(x-\om+\lambda)p_{\rho}(x)dx} = \bar{f};
 \end{cases}
$$
as proved in Appendix \ref{appd: relu_1}.

\begin{figure}
 \centering
\begin{tabular}{cc}
\includegraphics[width=0.5\linewidth]{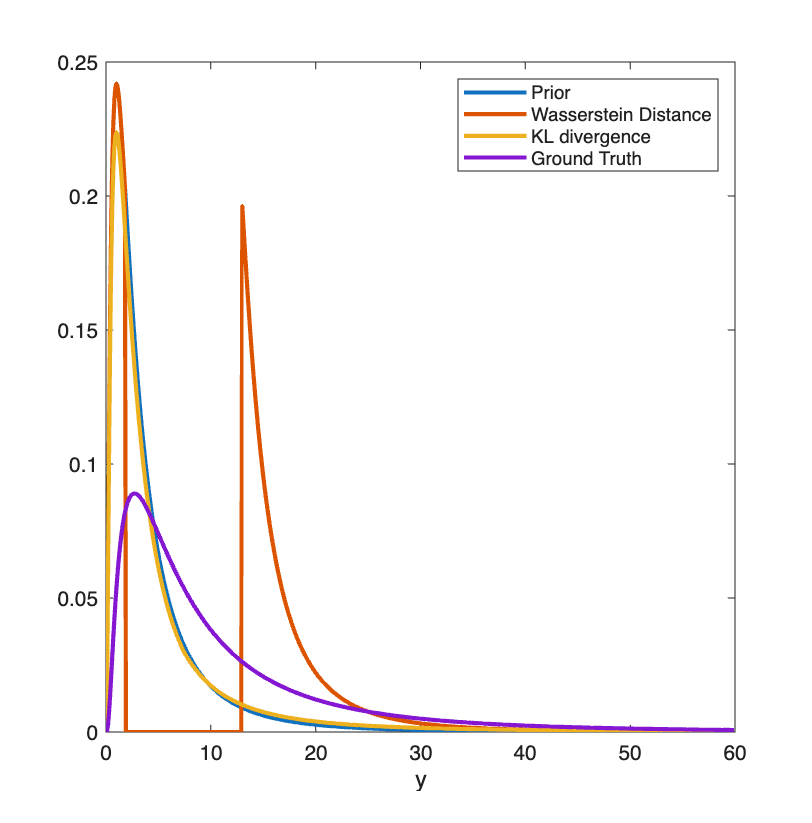}     &  \includegraphics[width=0.5\linewidth]{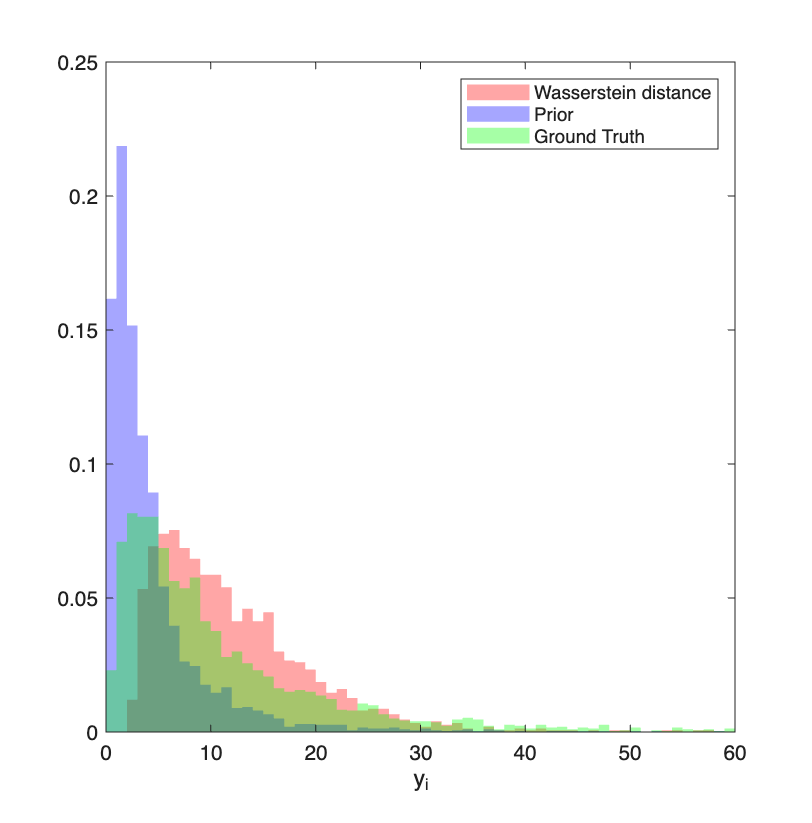} 
\end{tabular}
\caption{Left: the prior, the {surrogate} distribution, the target distribution restored from KL divergence and the exact target distribution restored from Wasserstein distance without smoothing; Right: the sample-based prior and the target distributions restored from Wasserstein distance with smoothing inequalities. The prior is $\text{Lognormal}(1,1)$. The \blue{surrogate} measure is $\text{Lognormal}(2,2)$. Single RELU constraint \eqref{func:relu} is adopted with $\om=7.3891$ and $\bar{f}_k$ computed via \eqref{eq:guassian_constraint_result}.  }
\label{fig:relu_1}
\end{figure}

\subsection{Multiple RELU functions}
Switching to cases of multiple expectation constraints, notice that these constraints may not be consistent, the simplest example being two constraints setting the expectation of the same function to two different values. Thus it is necessary to assume that there exists at least one distribution satisfying all of the constraints.

\begin{assume}[Assumption of consistency] 
 For problem \eqref{opt:constrained_ot_monge}, we assume that there exists a probability measure $\mu$ satisfying all the equality constraints:
$$
        \int_{Y}{f_{k}(y)d\mu} = \bar{f}_k,\;k=1,2,\dots,K.
$$
\end{assume}
\noindent
We guarantee that this condition holds in our synthetic examples, by picking a surrogate distribution $\mu$ for the target and adopting for $\{\bar{f}_k\}$ the expected values of the $\{f_k\}$ on it.

For the RELU functions, we have
\begin{equation}
    f_k(x) = \begin{cases}
    x - \om_k & x\geq \om_k, \\
    0 & \mbox{otherwise},
    \end{cases}
    \label{func:relu}
\end{equation}
where we assume for convenience that $\om_1<\om_2<\dots,<\om_K$. In finance, multiple RELU functions appear as constraints through the pricing of up-and-out options of an asset at different thresholding prices.

Similar to the single RELU case, the solution of the KL-based problem is
\begin{equation}
    p^{KL}_\mu(y) = 
    \begin{cases}
    \frac{1}{Z_K}p_\rho(y)& y<\om_1 ,\\
        \frac{1}{Z_K}p_\rho(y)e^{\sum_{i=1}^{k}{\lambda_i(y-\om_i)}} & y\in [\om_{k},\om_{k+1}], \\
         \frac{1}{Z_K}p_\rho(y)e^{\sum_{i=1}^{K}{\lambda_i(y-\om_i)}} & y>\om_K ,
    \end{cases}
\label{sol:mu_kl_divergence_relu}
\end{equation}
where $Z_K$ refers to the normalization coefficient:
\begin{equation}
    Z_K = \int_{\om_K}^{\infty}{p_\rho(y)e^{\sum_{i=1}^K{\lambda_i(y-\om_i)}}dy} + \sum_{k=1}^{K-1}{\int_{\om_{k}}^{\om_{k+1}}{p_\rho(y)e^{\sum_{i=1}^{k}{\lambda_i(y-\om_i)}}dy}} + \int_{-\infty}^{\om_1}{p_\rho(y)dy} 
    \label{eq:Z_K}
\end{equation}
and $\lambda_k,\;k=1,2,\dots, K$ meet the equality constraints:
\begin{equation}
    \frac{1}{Z_K}\int_{\om_k}^{\infty}{p_\rho(y)e^{\sum_{i=1}^k{\lambda_i(y-\om_i)}}(y-\om_k)dy} = \bar{f}_k,\;k=1,2,\dots,K.
    \label{sol:lambda_k}
\end{equation}
The solution for the OT-based problem is more complex; it can be best presented in a backward recursive way:
\begin{equation}
         \mu^{OT}(A) = \rho(T^{-1}_*(A)),
\end{equation}
  where $T_*(x) = x + \lambda_{K*} \mathbf{1}_{[ x_{K*},\infty)}(x) + \sum_{k=1}^{K-1}\tau_{k*}(x)\mathbf{1}_{[x_{k*},x_{k+1*}]}(x)$. Here $\lam_{K*}, x_{K*}$ are solutions of one of the following equations:

1-1) 
 \begin{equation}
 \begin{cases}
x_{K*}=\om_K-\lam_{K*}, \\
      \int_{\om_K-\lam_{K*}}^{\infty}{(z-\om_K+\lam_{K*})p_{\rho}(z)dz} = \bar{f}_K;
 \end{cases}
\end{equation}

1-2)
\begin{equation}
    \begin{cases}
      x_{K*} =\om_K - \frac{\lam_{K*}}{2}, \\
      \int_{x_{K*}}^{\infty}{(z-\om_K+\lam_{K*})p_\rho(z)dz} = \bar{f}_K.
    \end{cases}
\end{equation}
And for $k=K-1,K-2,\dots,1$,
\begin{equation}
    \tau^*_k(x) = 
    \begin{cases}
        \om_{k+1} - x,  & x\in [\om_{k+1} - \lam_{k*},x_{k+1*}],\\
        \lambda_{k*}, &  x \in [\om_k-\lam_{k*} ,\om_{k+1}-\lam_{k*}],\\
        \om_k-x & x\in [x_{k*},\om_k-\lam_{k*}],
    \end{cases}  \label{eq:tau_final_solution_case1}
\end{equation}
Defining
\begin{equation}
\begin{aligned}
\Delta \om_k & = \om_{k+1}-\om_k;\\
    H^+_k(x,\lam) &= \int_{x}^{x_{k+1*}}{(z-\om_k+\lam)p_\rho(z)dz};\\
    H^{-}_k(x,\lam) & =\int_{x}^{\om_{k+1}-\lam}{(z-\om_k+\lam)p_\rho(z)dz} + \Delta \om_k\int_{\om_{k+1}-\lam}^{x_{k+1*}}{p_\rho(z)dz};\\
    \Delta \tilde{f}_k(\lam)& = \bar{f}_k- \bar{f}_{k+1} + \Delta \om_k\int_{x_{k+1*}-\lam}^{\infty}{p_\rho(z)dz},
\end{aligned}
\end{equation}
the parameters $\lambda_{k*}$ and $x_{k*}$ satisfy one set of the following system equations: 
\begin{equation}
\begin{aligned}
 \text{2-1)} \qquad &x_{k*} = \om_k -\frac{1}{2}\lam_{k*},\;H^+_k(x_k,\lam_k) = \Delta\tilde{f}_k; \\
 \text{2-2)} \qquad    &x_{k*} = \om_k -\frac{1}{2}\lam_{k*},\;\;H^{-}_k(x_k,\lam_k)= \Delta\tilde{f}_k;\\
 \text{2-3)} \qquad  & x_{k*} = \om_k -\lam_{k*},\;\;H^+_k(x_k,\lam_k) = \Delta\tilde{f}_k; \\
 \text{2-4)} \qquad    &x_{k*} = \om_k -\lam_{k*},\;\;H^{-}_k(x_k,\lam_k)= \Delta\tilde{f}_k; \\
 \text{2-5)} \qquad& \Delta\om_k\int_{x_{k*}}^{\infty}{p_\rho(z)dz} = \Delta f_k,\; \lam_k\leq \om_{k+1}-x_{k+1},
\end{aligned}
\end{equation}
as proved in appendix \eqref{appd:relu_K}.

\begin{figure}[t!]
    \centerline{\subfloat{\includegraphics[width=0.5\linewidth] {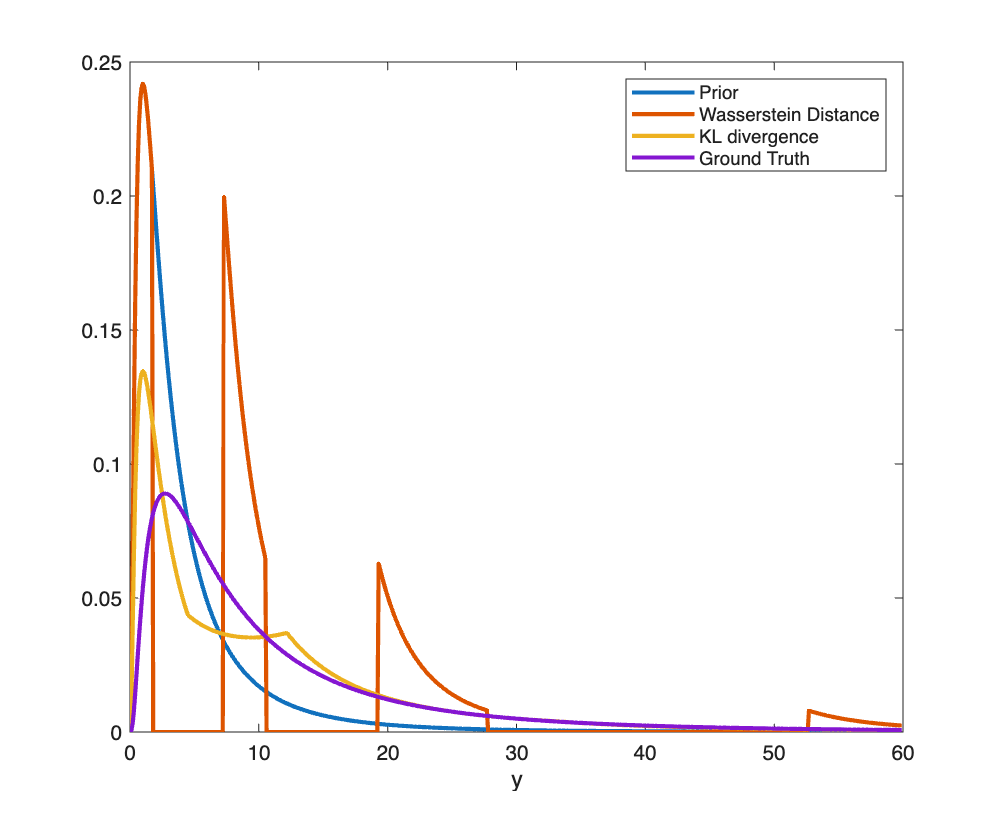}}
   \subfloat{\includegraphics[width=0.5\linewidth] {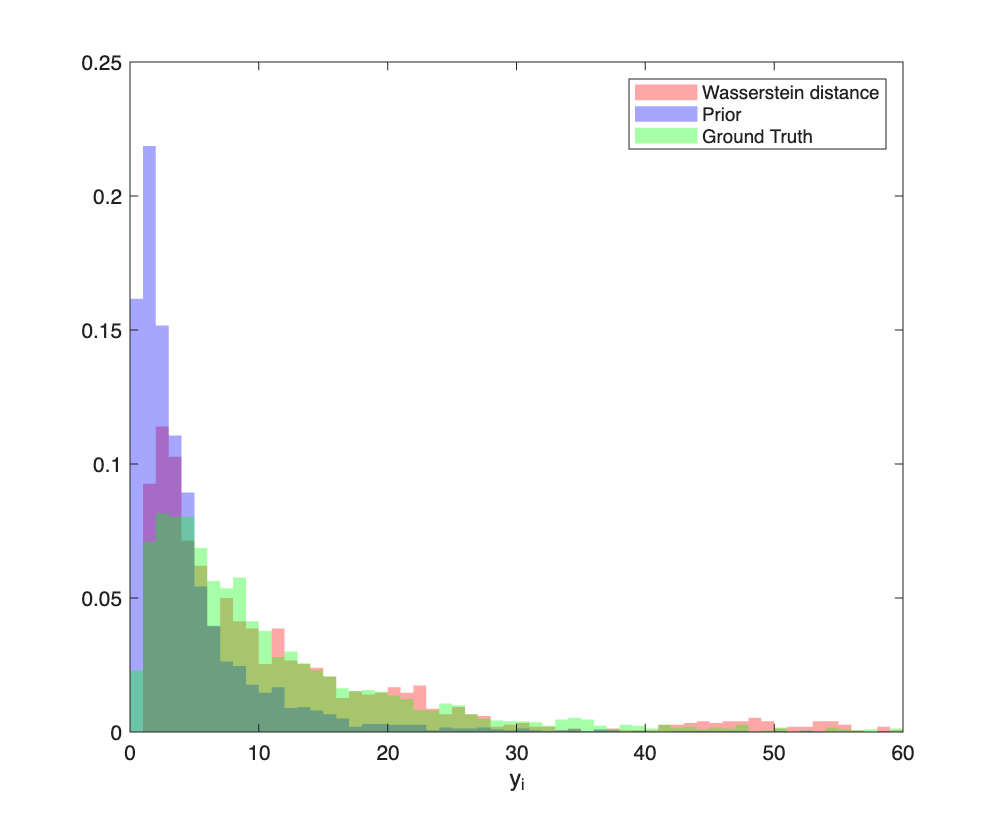}}
    }
    \caption{Left: Target distributions resulting from minimizing KL divergence as well as Wasserstein distance. The prior is $\text{Lognormal}(1,1)$. The {surrogate} measure is $\text{Lognormal}(2,2)$. We select $K=3$ with $\bar{f}_k$ $(k=1,2,\dots, K)$ computed in \eqref{eq:guassian_constraint_result}. 
   Right: the finite-sampled counterpart with inequality constraints. We choose $N=2000$ samples from the prior distribution and compute the target samples.}
    \label{fig:relu_k}
\end{figure}

In order to illustrate the different results for the KL divergence and the OT methods, we adopt as prior $p_\rho = \text{Lognormal}(0,1)$ and as surrogate target $p^{K}_\mu = \text{Lognormal}(1,1)$, under which we compute the prices of four options,  $ K=5,\;\om_k= -3+k,\;k=1,2,\dots,K$ and 
 \begin{equation}
     \bar{f}_k = \int_{\om_k}^{\infty}{(y-\om_k)p^{K}_\mu(y)dy}.
    \label{eq:guassian_constraint_result}
 \end{equation}
The resulting target distributions are displayed in the left of Figure \ref{fig:relu_k}. 
 
We observe that the KL divergence approach tends to ``stretch" the prior distribution vertically, resulting in sharp corners as artifacts, while the constrained optimal transport framework shifts samples horizontally, leading to accumulations and gaps as artifacts. In the next section, we will explore strategies to mitigate these artifacts through additional constraints.

\section{Smoothing inequality constraints}

The examples above make it clear that, when the functions $\{f_k\}$ are non-smooth, the optimal target distribution $\mu$ is often singular in two typical ways: assigning finite measure to small sets or having gaps, i.e. areas where $\mu$ vanishes surrounded by others where it does not. These are not artifacts: for instance, since the map with smallest cost transporting a prior distribution to a subset of $\mR^n$ needs not move any mass beyond the subset's boundary, it will accumulate finite mass along it, and if only a fraction of the original mass is required to move to said subset, then the most distant part of the prior will stay in place and a gap will form between it and the transported fraction. Yet these features may be undesirable: we may prefer the mass to be somewhat spread across the target subset, and regions with zero probability may be unrealistic models for the problem under consideration.

In order to mitigate such effects, we introduce the following two inequalities on the target probability density $p_\mu$:
\begin{align}
    \int_{X}{p^2_\mu(y)dy}\leq {M}_{\delta} ,\; \int_{D}{\frac{1}{p_\mu(y)} dy}\leq {M}_0 ,
        \label{ineq:bound_f_mu}
\end{align}
where the first inequality forbids measure accumulation in small sets and the second restricts areas of small density within a specified domain $D$.

\section{A numerical algorithm}
\label{sec:algorithm_design_finite_sample}

The previous sections established a theoretical framework for optimal transport with expectation constraints and provided several examples where explicit solutions under special choices of $f_k,\bar{f}_k$ and $p_\rho$ can be found.  However, explicit analytical solutions for the general problem in \eqref{opt:constrained_ot_monge} are typically not available, particularly in the presence of the smoothing constraints in (\ref{ineq:bound_f_mu}), which are nonlinear in $p_{\mu}$. Moreover, the prior $\rho$ is often not provided in closed form but through a set of independent samples $\{x_i\}$.

In this section, we develop a numerical algorithm that uses the $n$ samples $\{x_i\} \sim \rho$ as data and provides as output a corresponding set $\{y_i = T^*(x_i)\}$ of samples of the estimated distribution $\mu$. 
We replace the  prior measure $\rho$ by its empirical counterpart $\hat{\rho}$:
\begin{equation}
    \hat{\rho} = \frac{1}{n}\sum_{i=1}^n{\delta_{x_i}},
\end{equation}
and seek target samples $\{y_i\}_{i=1}^n$, such that the empirical target measure
\begin{equation}
    \hat{\mu}^* = \frac{1}{n}\sum_{i=1}^n{\delta_{y^*_i}},
\end{equation}
satisfies the constraints and the empirical transport cost is minimized:
\begin{equation}
 \underset{\{y_i\}}{\min}\;   L = \frac{1}{n}\sum_{i=1}^n{\|x_i-y_i\|^2} + \sum_{k=1}^{K}{\lam_k\Big(\frac{1}{n}\sum_{i=1}^n{f_k(y_i)-\bar{f}_k \Big)^2}},
    \label{eq: objective_hard_constraint}
\end{equation}
where we have replaced the constraints by a quadratic penalty in the objective function with penalization parameters $\{\lambda_k\}$. This minimization problem can be solved easily through gradient descent if the functions $\{f_k\}$ are smooth and their support is broad enough (More on this below.) When they are not, as in most of our examples above, we mollify them into $\{f^{\eps}_k\}$, smooth functions with larger support that converge to the $\{f_k\}$ as $\eps\rightarrow 0$.

It is not solely to have well-defined derivatives of $L$ that we mollify the $\{f_k\}$: the vanishing gradients of functions $f_k$ within flat regions can stall the optimization. Consider as a simple example the Heaviside step function
$$ f(x) = \begin{cases} 0 & \hbox{for $x < 0$} \cr 1 & \hbox{for $x \geq 0$},  \end{cases} $$
which we can use to require a given fraction $\bar{f}$ of the mass of $\mu$ to lie to the right of $x=0$. The problem of using functions like this within (\ref{eq: objective_hard_constraint}) is not that $f'(0)$ is not well defined, but rather that $f'(x) = 0$ for all $x \ne 0$. Because of this, nothing in $\frac{\partial L}{\partial y_i}$ pushes any point $y_i$ to the other side of $y=0$, so all samples will remain at their original position $y_i = x_i$, which minimizes the transportation cost but does not satisfy the constraint. By contrast, the mollified
$$ f^{\eps}(x) = \frac{1}{2}\left[\tanh\left(\frac{x}{\eps}\right)+1\right] $$
does not have this problem if $\eps$ is large enough for $f^{\eps}$ to be sensitive to the values of all $\{y_i\}$. 

Motivated by this, we employ an annealing-like schedule. We start with a relatively large $\eps$ to ensure far from zero gradients across sample space. After the algorithm converges for a given $\eps$, we reduce it (e.g., 
$\eps\leftarrow \eps/\beta$ for some 
$\beta>1$) and repeat the optimization. Gradually reducing the mollification allows the solution to satisfy the constraints on the true $\{f_k\}$ without leaving sample-points behind.

\paragraph{Smoothing inequality constraints} 
In order to enforce the inequality constraints in (\ref{ineq:bound_f_mu} ), we need to estimate the density function $p_{\mu}$ from the samples $\{y_i\}$ and evaluate the integrals via Monte Carlo. We choose to do this through kernels, writing
$$ p_{\mu}(y) \leftarrow \frac{1}{n} \sum_{j=1}^n K\left(y, y_j\right), $$
$$ \int p_{\mu}^2(y) \ dy = \int p_{\mu}(y) \ d\mu \leftarrow \frac{1}{n} \sum_{j=1}^n p_{\mu}(y_j) , $$
$$
    \int_{D}{\frac{1}{p_\mu(y)} dy} =  \int_{Y}{\frac{\mathbf{1}_D}{p^2_\mu(y)} d\mu}  \leftarrow \frac{1}{n} \sum_{j=1}^n \frac{\varphi_{\eps}(y_j;D)}{\hat{p}^2_{\mu}(y_j))},
$$
where $\varphi_{\eps}(\cdot\ ;D)$ is  a mollifier of the indicator function of $D$.
For our examples, we have used Gaussian kernels $K$ with bandwidth determined through the rule of thumb applied to the original $\{x_i\}$.

Then, using barrier functions for the inequality constraints, the full objective function becomes
\begin{equation}
 L^{\eps,t} = \frac{1}{n}\sum_{i=1}^n{\|x_i-y_i\|^2} + \sum_{k=1}^{K}{\lam_k\Big(\frac{1}{n}\sum_{i=1}^n{f^{\eps}_k(y_i)-\bar{f}_k \Big)^2}} + \frac{\lambda_{\delta}}{t} B_{\delta} + \frac{\lambda_0}{t} B^{\eps}_0,
    \label{eq: objective_full}
\end{equation}
where
$$ B_{\delta} = -
    \log\left(-\frac{1}{n^2 h}\sum_{i,j}{K(\frac{y_i-y_j}{h})}+ {M}_{\delta}\right),  $$
$$ B^{\eps}_0 = -\log\left(-n h^2 \sum_{i}\frac{\varphi_{\eps}(y_j;D)}{(\sum_{j}{K(\frac{y_i-y_j}{h})})^2}+ {M}_0\right).$$

\paragraph{Adaptive learning rate} 
We use a backtrack line search to determine the step size $\eta^{(t)}$ at each iteration, ensuring sufficient decrease and robust convegence. For each update step, we begin with an initial guess of the learning rate slightly expanded from the size of the previous step: $\eta^{(t)}\leftarrow (1+\delta)\eta^{(t-1)}$ and reduce it until the Armijo-Goldstein condition is satisfied: 
\begin{equation}
    L^{\eps}(y^{(t)}-\eta^{(t)}\nabla L^{\eps}(y^{(t)}) )\leq  L^{\eps}(y^{(t)}) +\alpha\eta^{(t)} \|\nabla  L^{\eps}(y^{(t)})\|^2,
    \label{alg:step_size_test}
\end{equation}
where $\alpha\in (0,0.5]$ is a control parameter. This ensures that each step provides a meaningful reduction in the objective function. 

We detail the algorithm for soft-constrained problems in Algorithm \ref{alg: sol_gradient_descent_unconstrained}.
\begin{algorithm}
    \begin{algorithmic}[1]
       \State \textbf{Input:} Parameters $\{\lambda_k\}^{K}_1,n,\{f^{\eps}_{k}\}^{K}_1$. Horizon $T_{\max}\in \mathbb{Z}$, tol.
       \State Collect samples $x = (x_1,\dots,x_n)\in\mR^{n}$ from the prior $\rho$.
       \State {\textbf{Initialize:}} $y^{(0)}$. {\textbf{Set:}} $\eps\leftarrow \eps_0$.
       \For{$j = 1$ \textbf{to} $J$} \Comment{Outer loop for annealing $\epsilon$}
\For{$t = 1$ \textbf{to} $T_{max}$} \Comment{Inner loop for gradient descent}
\State Compute gradient: $g^{(t)} \leftarrow \nabla_y L^{\epsilon}(y^{(t)}; x)$.
\State Determine step size $\eta^{(t)}$ using backtracking line search satisfying the Armijo condition.
\State Update samples: $y^{(t+1)} \leftarrow y^{(t)} - \eta^{(t)} g^{(t)}$.
\If{$\|y^{(t+1)} - y^{(t)}\|<\text{tol}$}
\State \textbf{break} \Comment{Converged for the current $\epsilon$}
\EndIf
\EndFor
\State Update starting point for next round: $y^{(0)} \leftarrow y^{(t+1)}$.
\State Anneal smoothing parameter: $\epsilon \leftarrow \epsilon / \beta$.
\EndFor
\State \textbf{Return:} Final target samples $y = (y_1, \dots, y_n)$.
    \end{algorithmic}
    \caption{Solution to the soft-constrained optimization}
     \label{alg: sol_gradient_descent_unconstrained}
\end{algorithm}

\subsection{One-dimensional examples}
The restored target measures in a sample-based setting using the Algorithm \ref{alg: sol_gradient_descent_unconstrained} are depicted on the right in Figures \ref{fig:relu_1} and \ref{fig:relu_k}. We observe that the smoothing inequalities help mitigate gaps and accumulations by re-distributing the measures located at accumulations and spreading them out to fill in the gaps.

\subsection{Two-dimensional examples}
We apply next the constrained optimal transport framework with smoothing inequalities on 2D examples, i.e. with $X=\mR^2$. For demonstration, we select for the equality constraints a single ($K=1$) two-dimensional indicator function supported on the complement of a set $D$: 
\begin{equation}
    f(z_1,z_2) = \mathbf{1}_{\mR^2\backslash D} =\begin{cases}
        1 & (z_1,z_2)\notin D \\
        0 & (z_1,z_2)\in D 
    \end{cases},\;\;\bar{f} = 0.
\end{equation}

\paragraph{Indicator function within a circle}
As a first example, $D = D_R = \{(z_1,z_2)\in \mR^2:\;z^2_1 +z^2_2\leq R^2\}$ specifies a circle within which the support of the target measure is required to lie. 
The OT solution without smoothing inequalities is 
\begin{equation}
    \mu^{OT} = \rho|_{D_R}+ \alpha_R  \delta_{\partial D_R}, \quad \alpha_R = \iint\limits_{z^2_1+z^2_2\geq R^2} f_\rho(z_1,z_2) \, dz_1 \, dz_2 .
    \label{eq:exact_measure_circle}
\end{equation}
%
Figure \ref{fig:result_2d_indicator} depicts the exact target measure \eqref{eq:exact_measure_circle} and the restored target measures with and without smoothing constraints.

\paragraph{Indicator function on a half plane:} As another example, we consider the half plane represented by the set $D = D_3 = \{(z_1,z_2): z_1\geq R\}$. 
The OT solution without smoothing inequalities is 
\begin{equation}
    \mu^{OT} = \rho|_{D_3} +\alpha_3 \delta_{\partial D_3}, \quad \alpha_3 = \iint\limits_{z_1\leq R} f_\rho(z_1,z_2) \, dz_1 \, dz_2 .
\end{equation}
%
Figure \ref{fig:2d_circle} depicts the exact target measures and the estimated target measure with and without smoothing constraints. We observe that the unregularized estimated target measure is close to the exact measure both in terms of its continuous and its atom components, while the regularized target measure mollifies some of the target measures near the boundaries and pushes it further away.

\begin{figure}
    \centering
    \begin{tabular}{ccc}
 \includegraphics[width=0.3\textwidth]{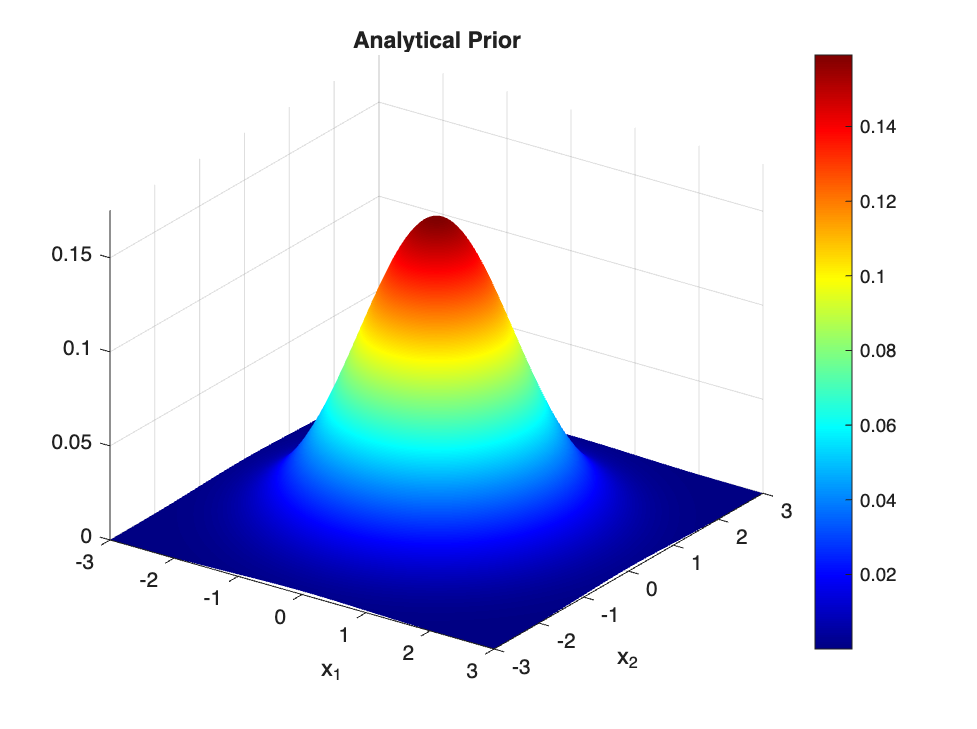}      &\includegraphics[width=0.3\linewidth]{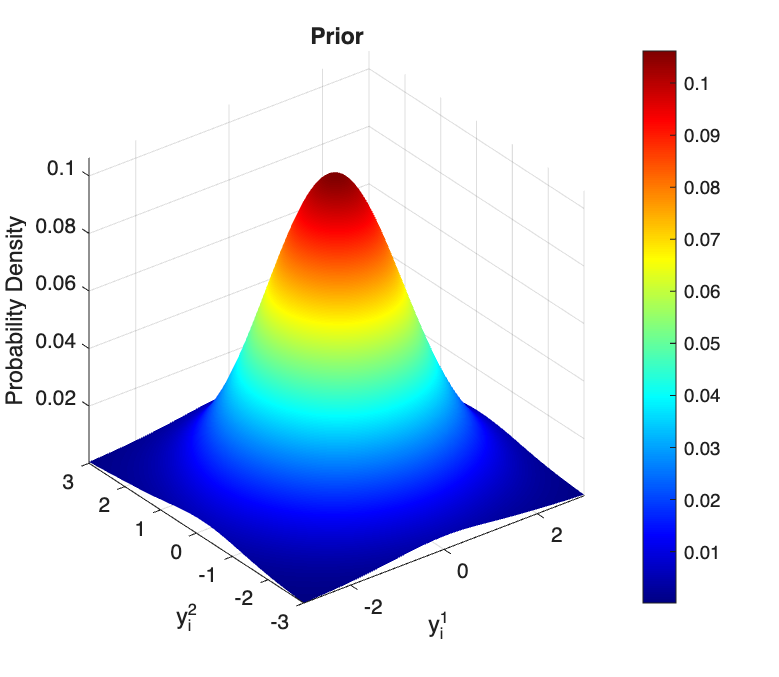}   & \includegraphics[width=0.3\linewidth]{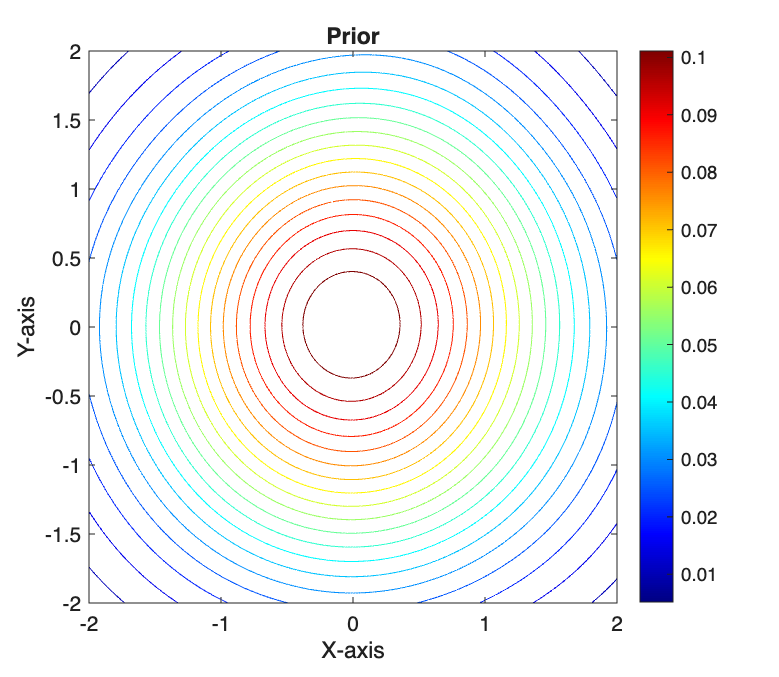} \\
 \includegraphics[width=0.3\textwidth]{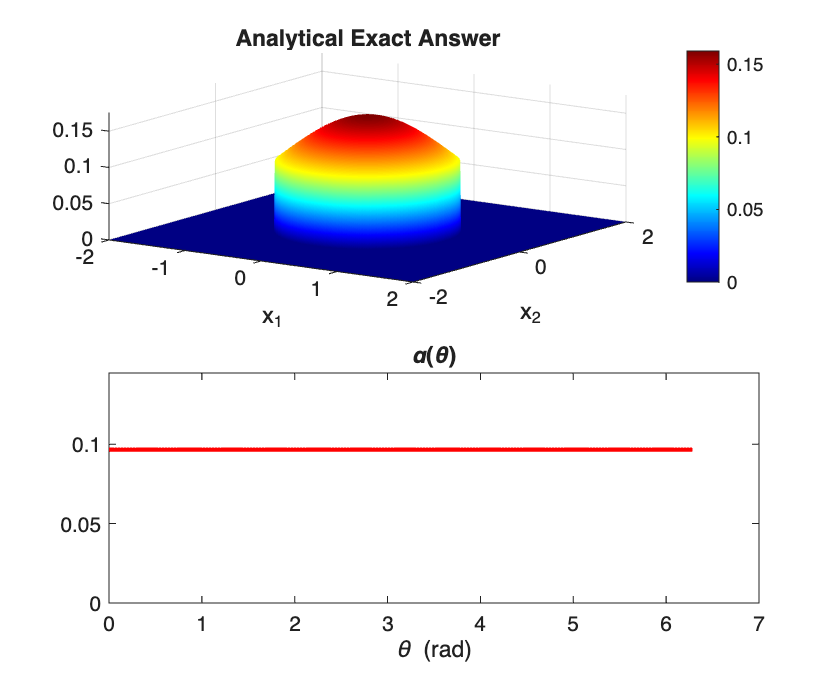}          &\includegraphics[width=0.3\linewidth]{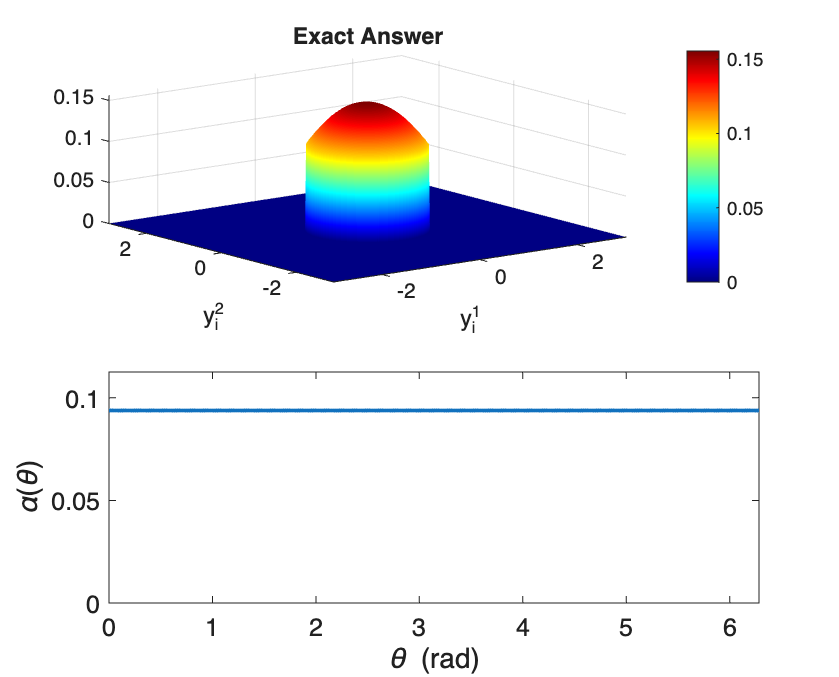}   & \includegraphics[width=0.3\linewidth]{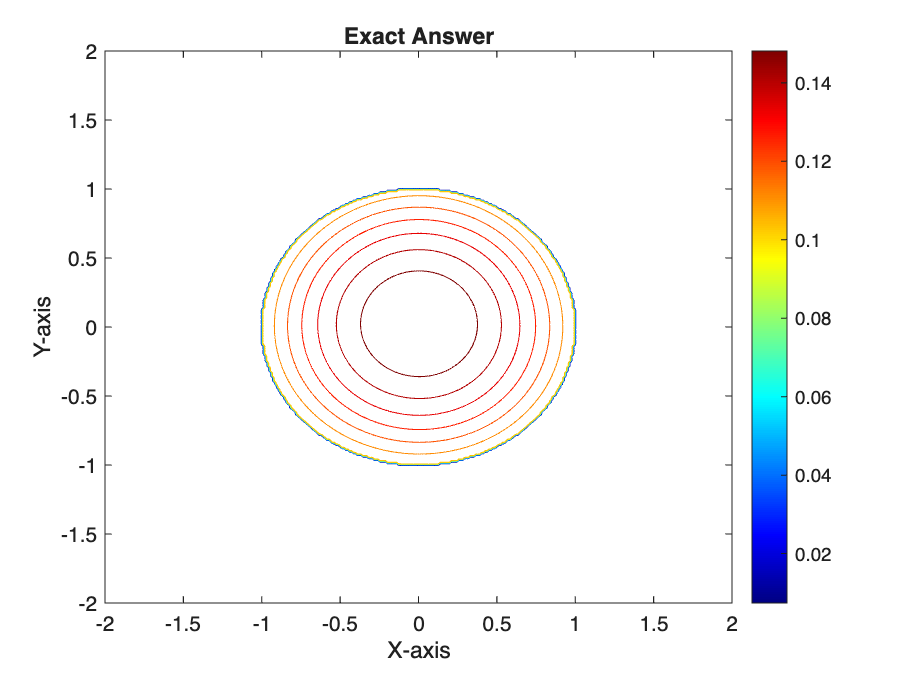} \\
&\includegraphics[width=0.3\linewidth]{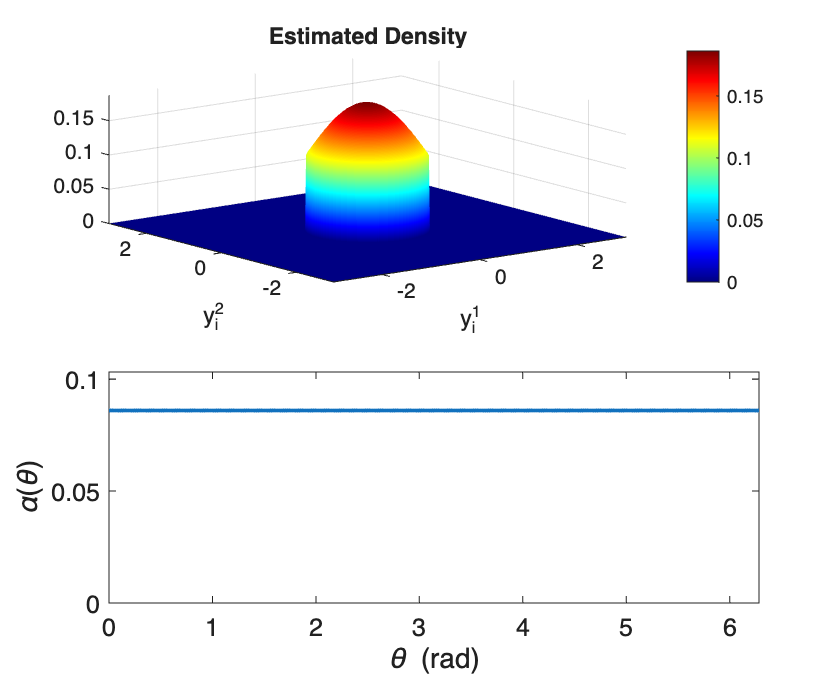}   & \includegraphics[width=0.3\linewidth]{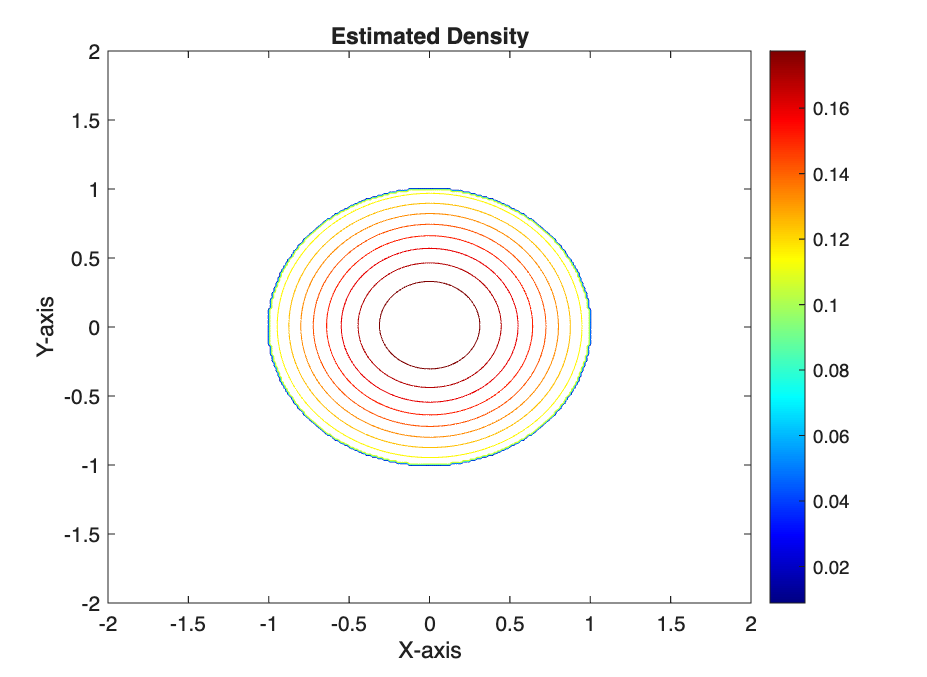} \\
&\includegraphics[width=0.27\linewidth]{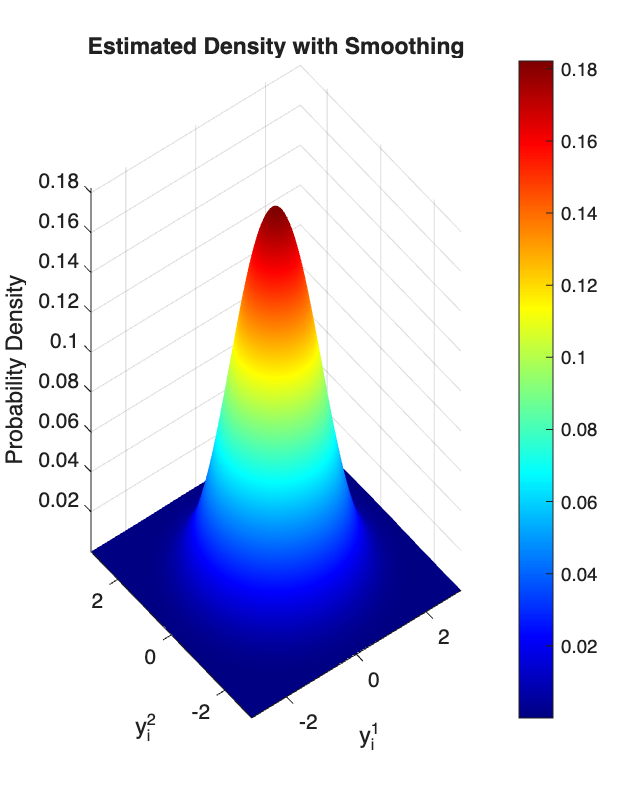}   & \includegraphics[width=0.27\linewidth]{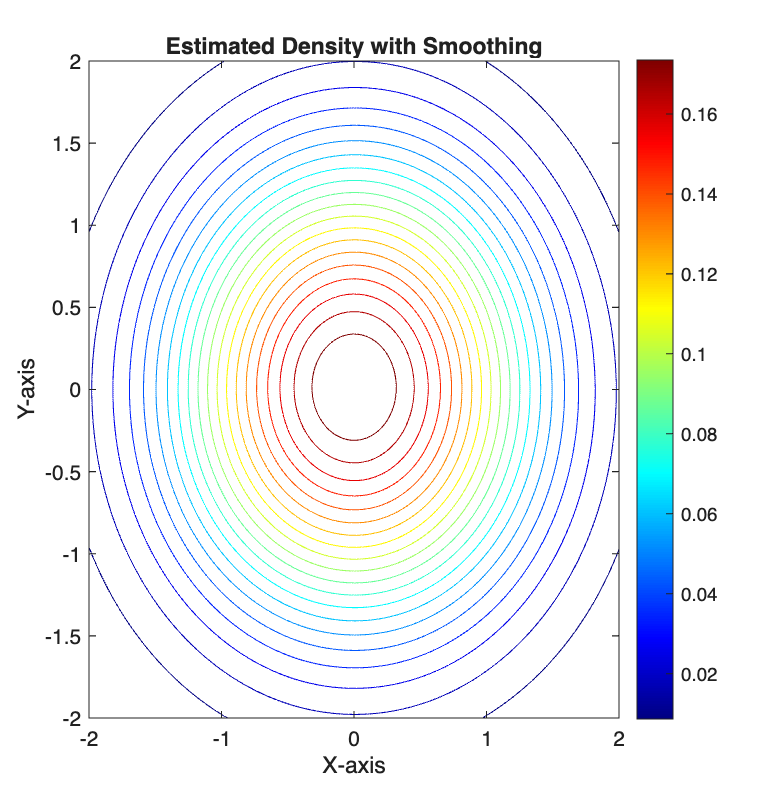} 
    \end{tabular}
\caption{The surface plots and contour plots of the prior, the exact solution and the estimated target measures using the proposed framework without and with regularization. All measures are constructed using samples via kernel density estimation techniques. The exact solution and estimated measure without regularization both contain a delta measure supported on the circle $\{(z_1,z_2)|z_1^2+z^2_2=R^2\}$, denoted as $\alpha(\tha)\delta(R,\tha)$ and are zero outside the circle.}
    \label{fig:result_2d_indicator}
\end{figure}

\color{black}

\begin{figure}
 \centering
\begin{tabular}{ccc}
\includegraphics[width = 0.25\linewidth]{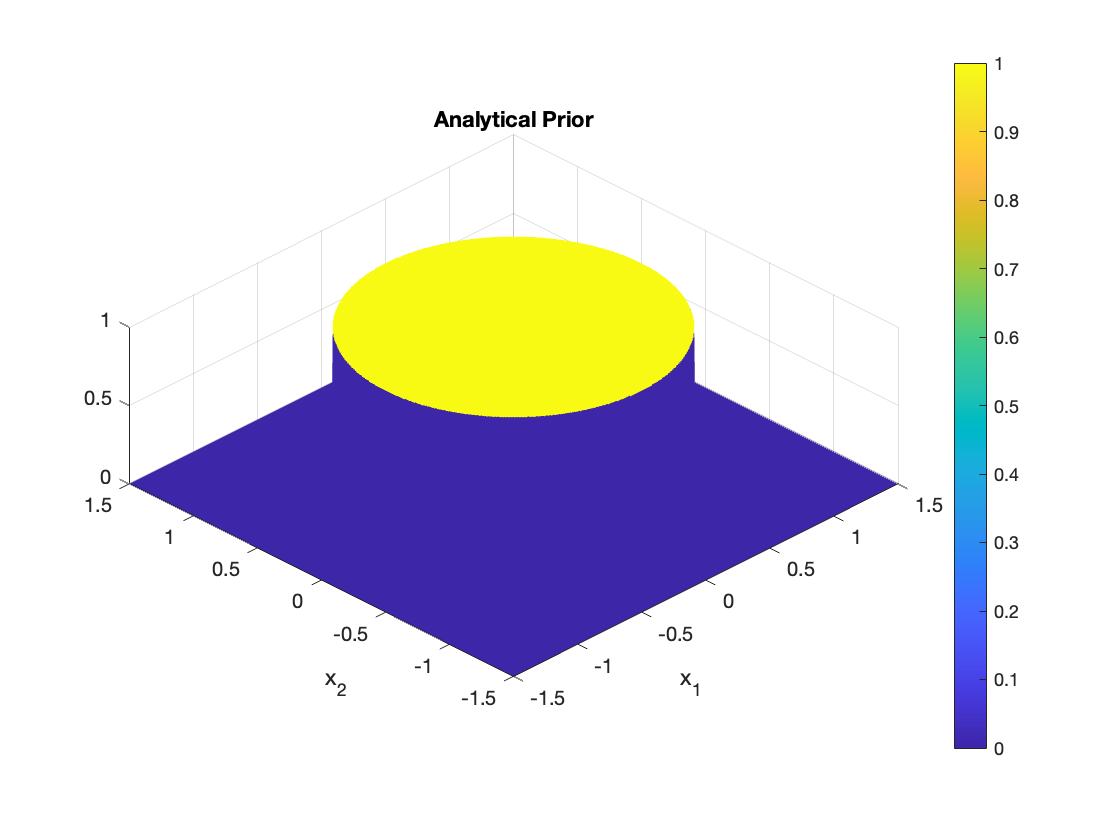}&\includegraphics[width=0.3\linewidth]{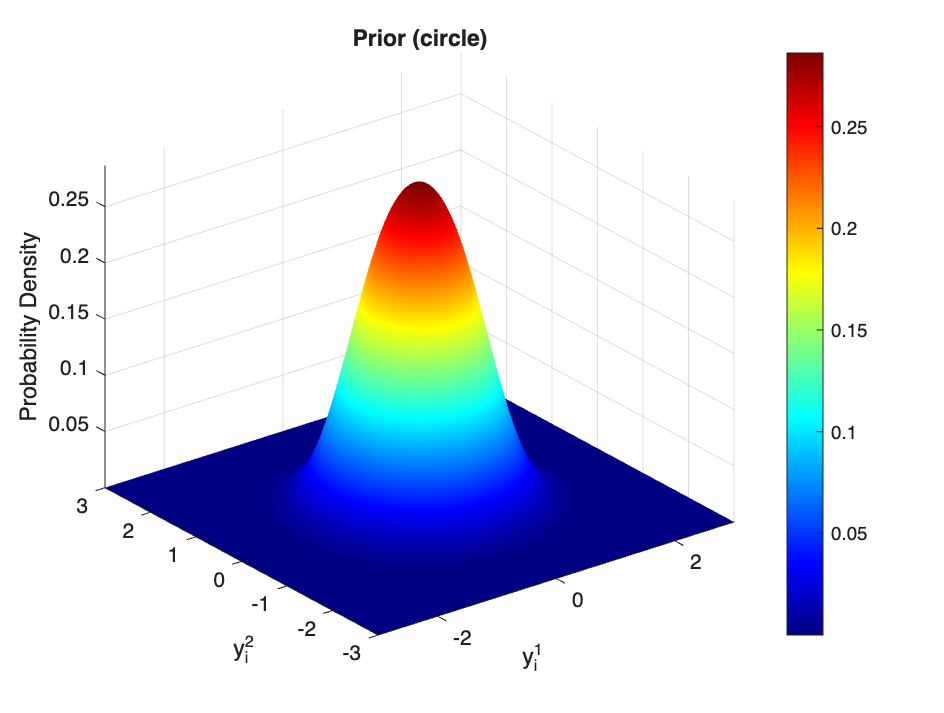}  & \includegraphics[width=0.25\linewidth]{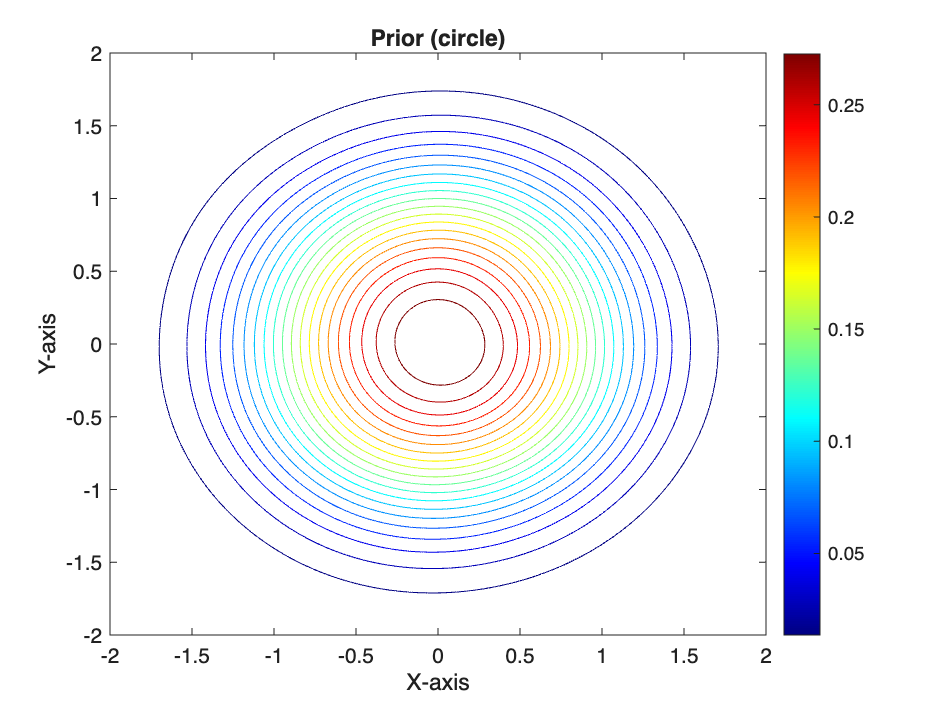} \\
\includegraphics[width = 0.25\linewidth]{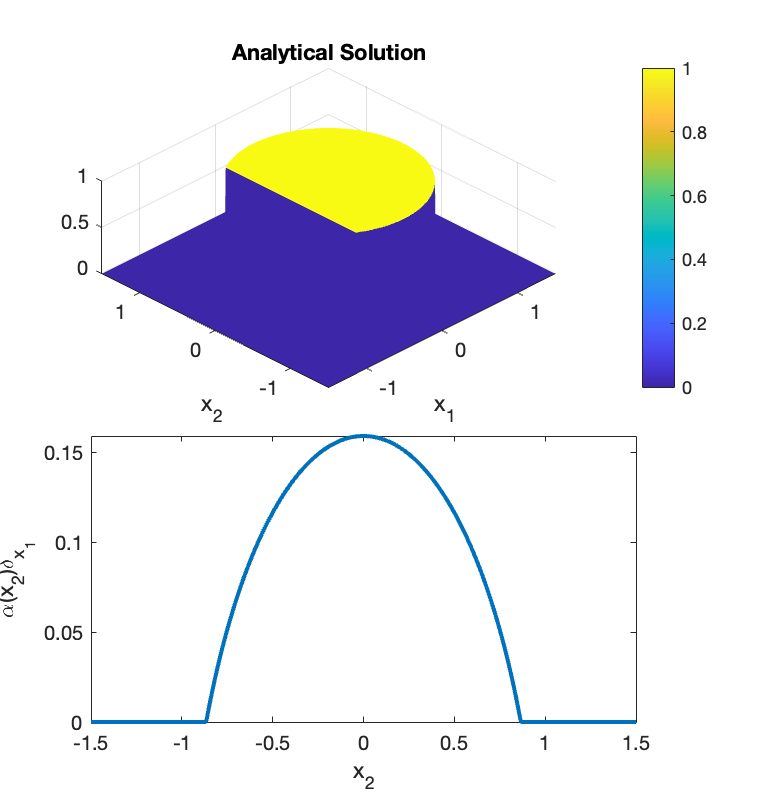}&
\includegraphics[width=0.3\linewidth]{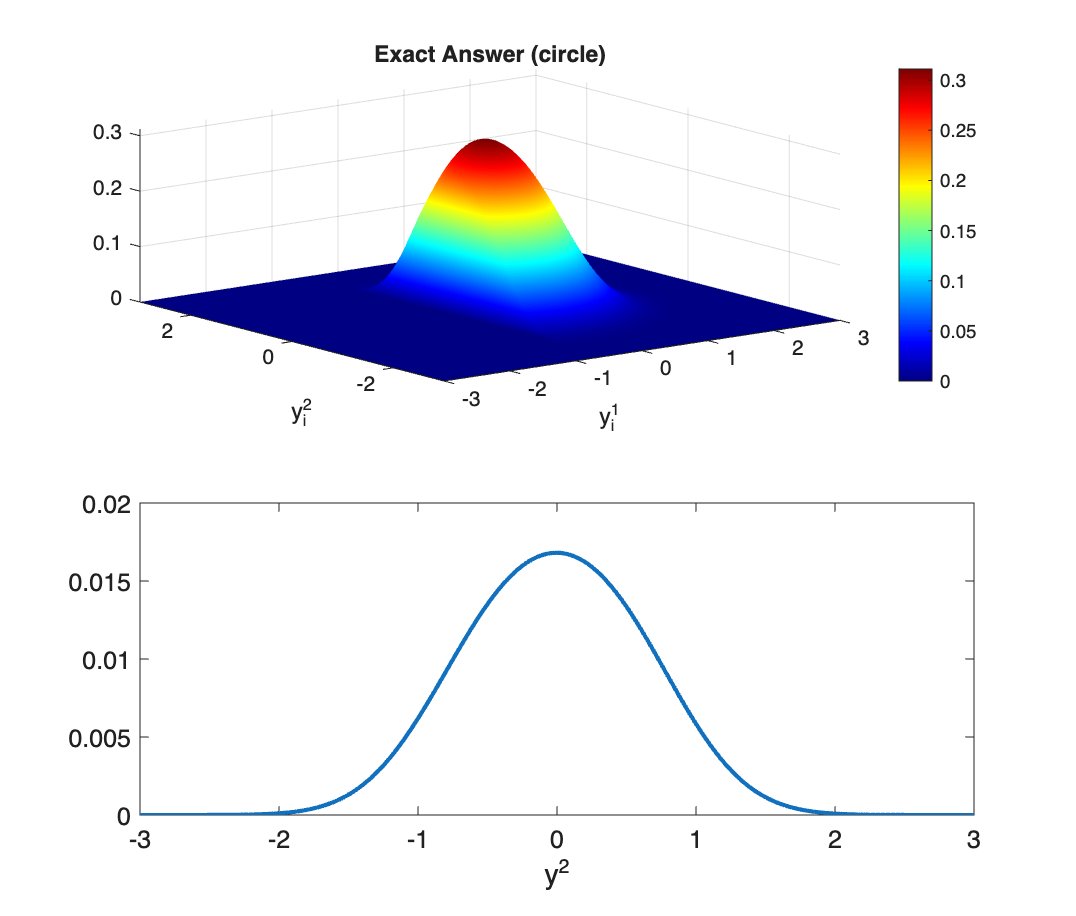}  & \includegraphics[width=0.3\linewidth]{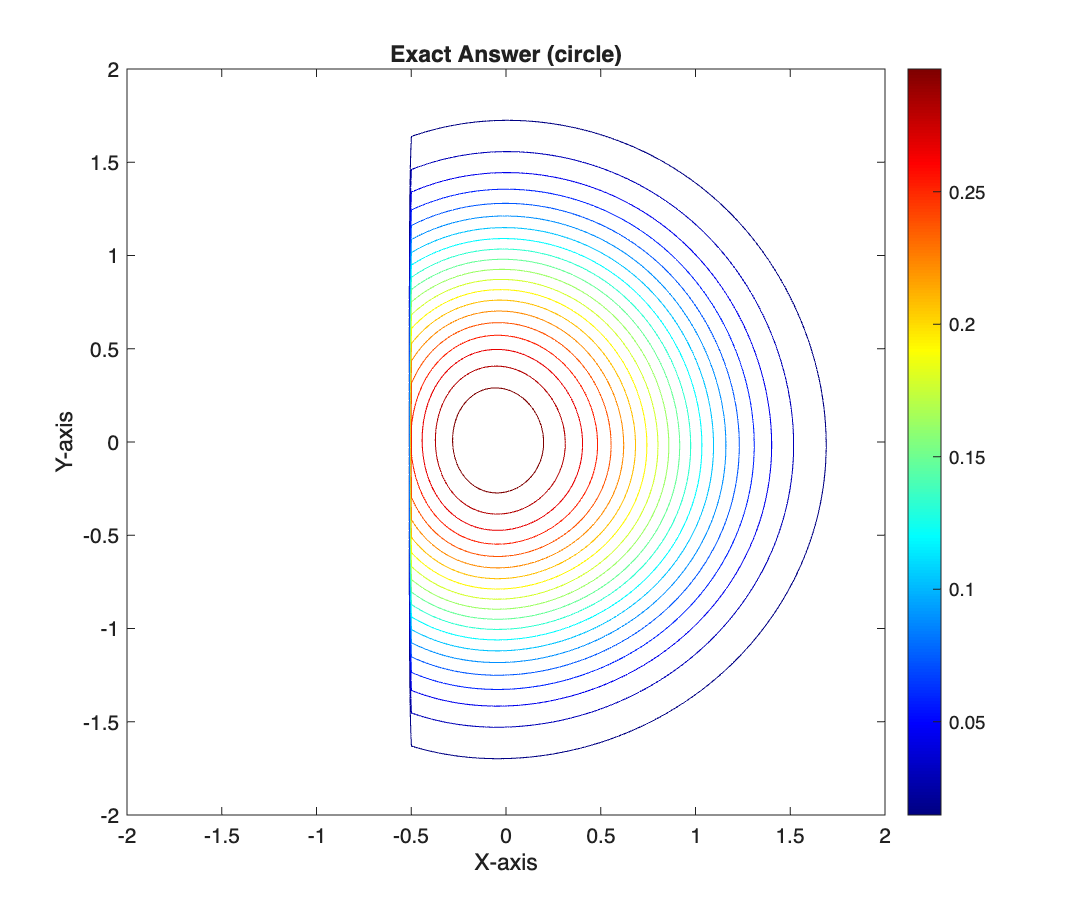} \\
&\includegraphics[width=0.3\linewidth]{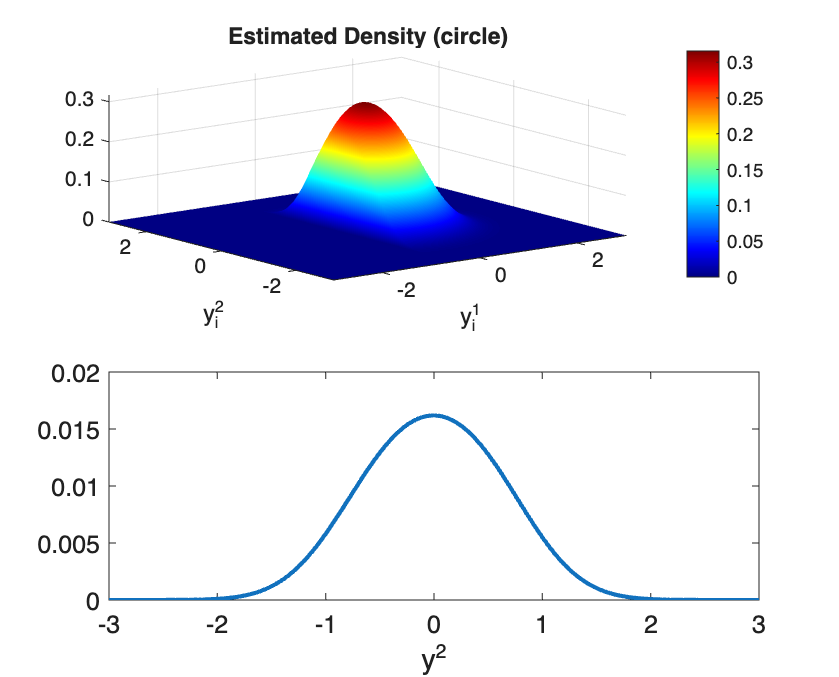}  & \includegraphics[width=0.3\linewidth]{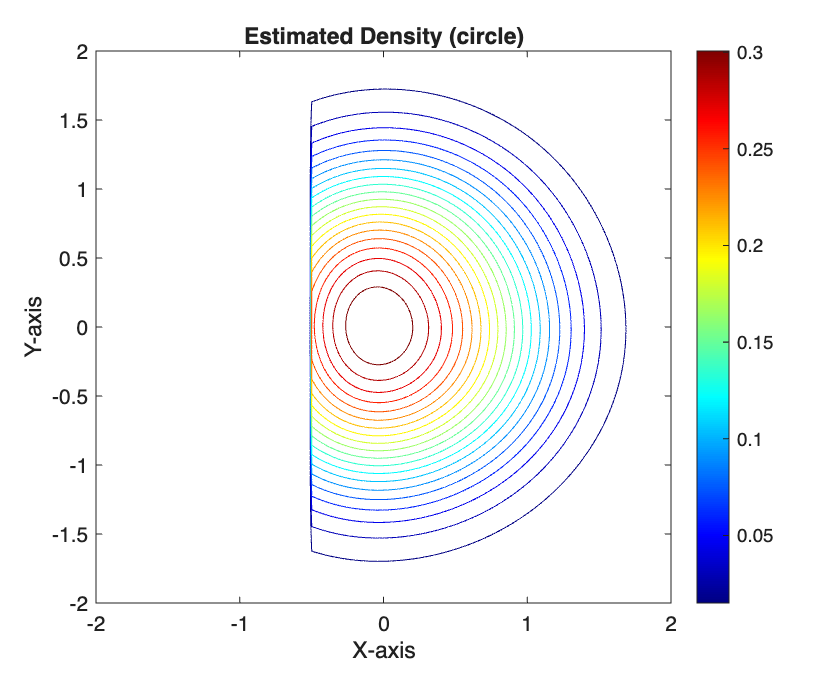} \\
&\includegraphics[width=0.3\linewidth]{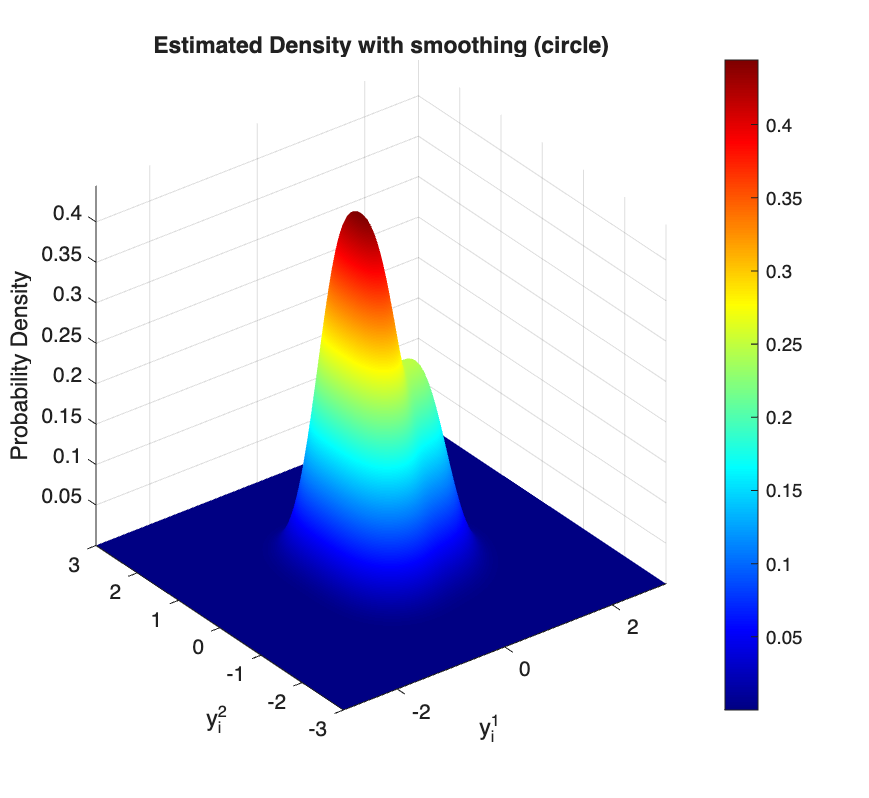}  & \includegraphics[width=0.3\linewidth]{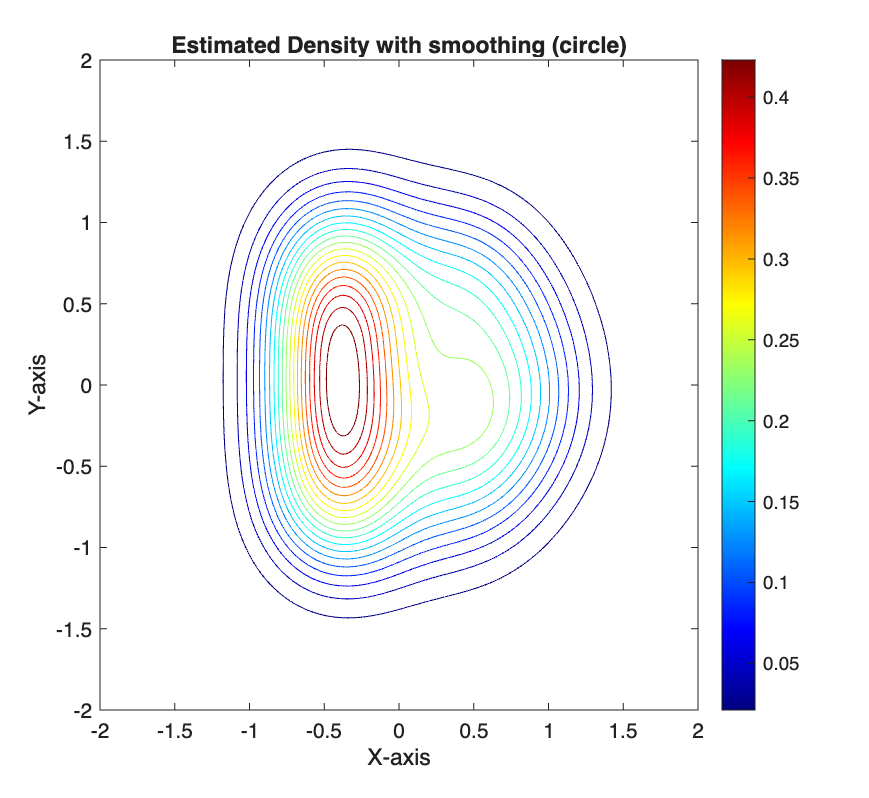} 
\end{tabular}
    \caption{The surface plots and contour plots of the prior, the exact solution and the estimated target measures using the proposed framework without and with regularization. All measures are constructed using samples via kernel density estimation techniques. The exact solution and estimated measure without regularization both contain a delta measure supported along the line$\{(z_1,z_2): z_1=R\}$ and are zero in $\{(z_1,z_2): z_1<R\}$ with $R=-0.5$. }
    \label{fig:2d_circle}
\end{figure}

\section{Case study}
\label{sec:case_study}
In this section, we apply the density estimation framework proposed in Section \ref{sec:formulation} to a proof-of-concept problem in quantitative finance, recovering the risk-neutral pricing measure of an underlying asset given market prices of a few vanilla options. This recovered measure is then used to price exotic options, whose market prices are assumed to be unknown. We use synthetic data to validate the framework, comparing its performance against a surrogate measure and the Kullback-Leibler (KL) divergence minimization approach \cite{avellaneda1997calibrating_KL_minimization}.

\begin{figure}[!t]
\centering
\centerline{
\subfloat{\includegraphics[width=0.5\linewidth]{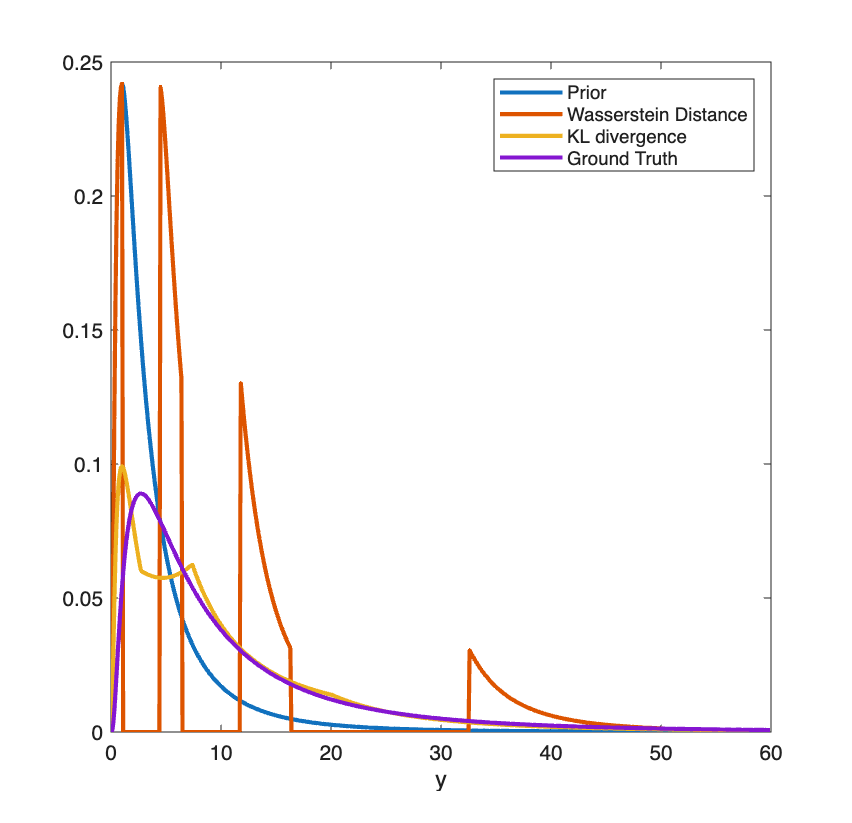}  }
\subfloat{
\includegraphics[width=0.5\linewidth]{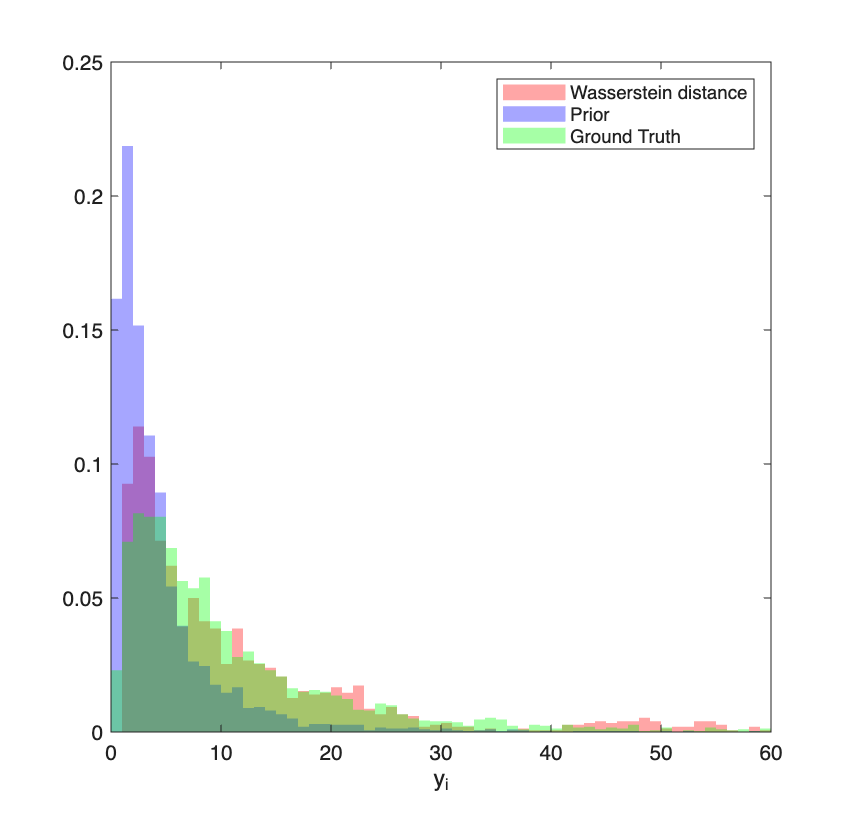} }}

    \caption{Left: the estimated pricing measure of the asset by the proposed framework and the surrogate measure. Right: the estimated histogram of the sampled target measure of the asset by the regularized framework, the surrogate  measure and the The prior measure. The prior is $p_\rho$ and $\bar{f}_k$ are computed based on $p^{K}_\mu$ and threshold prices $\om_1,\dots\om_K$ with $K=3$.}
\label{fig:result_wasserstein_k_3}
\end{figure}

\color{black}
\paragraph{Background} An option \cite{bouzoubaa2010exotic_option} is a contract to buy (`call') or sell (`put') a specific quantity of an underlying asset or instrument at a specified strike price on or before a specified date. Vanilla options are simple call/put options with no special or unusual features that can be exercised when the strike place performs better than the underlying asset price does. Mathematically, a vanilla call option can be expressed as the single RELU function in \eqref{func:relu_1}, where $x>0$ is the asset price and $\om>0$ the strike price. Exotic options are more complex in terms of payment structure, expiration dates and strike prices. 

The valuation of exotic options has been a research focus in finance for decades \cite{avellaneda1999introduction_pricing_risk}. In order to estimate the value of exotic options on an underlying asset given only the pricing of its vanilla options, we proceed in two steps. We first use this article's constrained optimal transport framework in lieu of the KL divergence adopted in \cite{avellaneda1997calibrating_KL_minimization}, to estimate the  underlying risk-free measure, and then utilize this estimated pricing measure to price the exotic option. We compare the results with their counterparts using the prior measure, the one estimated via KL divergence method and the ground-truth underlying pricing measure of the asset.   

\paragraph{Synthetic examples}

We illustrate our model with several synthetic examples, where a surrogate  underlying pricing measure, based on which the values of various exotic options are calculated, is assumed. 

We assume that the surrogate  underlying pricing measure of an asset is a log-normal distribution with parameters $\mu_K = 2, \sigma_K=1$, or Lognormal(2,1). 
We are given the true pricing of $K=3$ vanilla call options as constraints, with logarithm of threshold prices $\om_k = -\frac{K-1}{2}+k+1, k=1,2,\dots,K$. The equality constraint corresponding to the vanilla call options are given in \eqref{func:relu_1} with $\bar{f}_k$ computed based on $p^K_\mu$: $\bar{f}_k = \int_{\om_K}^{\infty}{(y-\om_K)p^K_\mu(y)dy}$.

Choosing as prior measure of the price $p_\rho = \text{Lognormal}(1,1)$, we depict the results generated by the KL divergence method and by our formulations in Figure \ref{fig:result_wasserstein_k_3}. We can observe that our restored measure is closer to the true underlying price measure of the asset than the the one estimated by KL divergence method. 

\begin{table}[h!]
\begin{tabularx}{\textwidth}{l *{5}{>{\centering\arraybackslash}X}}
\toprule
& Prior   & Wasserstein & KL Divergence & Smooth Wasserstein & surrogate  \\
 \midrule
Up-and-out       & 2.7121  & 10.355      & 9.7641        & 10.313             & 10.053       \\
& \blue{0.6735} & 5.7158      & 5.3727        & 5.2194             & 5.7588       \\
\midrule
Cash-or-nothing  & 2.0000       & 3.3446      & 3.4580         & 3.6720              & 3.7319       \\
& 2.0000       & 3.3446      & 3.1443        & 3.1307             & 3.3645       \\
\midrule
Asset-or-nothing & 3.0988  & 12.053      & 11.057        & 11.554             & 11.297       \\
                 & 2.2408  & 10.125      & 10.034        & 11.169             & 10.178      \\
\bottomrule
\end{tabularx}
\caption{Estimation of pricing of several exotic options using different methods. The prior  $p_\rho = \text{Lognormal}(1,1)$ and the surrogate  pricing measure is $\text{Lognormal}(2,1)$. The three options $g_1,g_2,g_3$ are given in \eqref{eq:g1}\eqref{eq:g2}\eqref{eq:g3} respectively.}
\label{tab:pricing_options}
\end{table}

\paragraph{Pricing exotic options} We now apply the estimated pricing measure to price several exotic options: down-and-out call option, cash-or-nothing option and asset-or-nothing option.

A down-and-out call option $g_1$ refers to an exotic option that provides a payoff only when the price lies above a certain strike price {$s_1$} and payoff is the difference between the current price and the threshold {$H_0$} as expressed as follows 
\begin{equation}
    g_1(x) = \max(x-s_1,0)\mathbf{1}_{x\geq H_0}.
    \label{eq:g1}
\end{equation}
Here we select $H_0 = 20.0855, s_1 = 1.6487$ and $H_0 = 2.7183, s_1 = 2.1170$, respectively.  

A cash-or-nothing option refers to returning a fixed cash when the pricing of the asset exceeds the strike price, or zero if it does not. With $s_2$ being the strike price, the price of an aseet-or-nothing option $g_2$ can be expressed as 
\begin{equation}
    g_2(x) = C\mathbf{1}_{x\geq s_2},
        \label{eq:g2}
\end{equation}
with $C$ being the cash value. Here we select $C=4$ and $s_2 = 2.7381, 1.6487$ respectively.

  An asset-or-nothing option refers to returning the value of the asset when the pricing of the asset exceeds the strike price, or zero if it does not. With $\om$ being the strike price, the price of an aseet-or-nothing option $g_3$ can be expressed as 
\begin{equation}
  g_3(x) = x\mathbf{1}_{x\geq s_3}, 
        \label{eq:g3}
\end{equation}
where we select $s_3 = 7.3891, 4.4817$ respectively.


  We utilize our estimated pricing measures of the assets obtained by our method and the KL divergence method to compute pricing of the above-mentioned exotic options and compare them with the true pricing (i.e. pricing computed from the surrogate  pricing of asset measures $p_\mu$). Table \ref{tab:pricing_options} lists the computed prices, and Table \ref{tab:relative_error_option} shows the relative errors of each estimate.
The results in Table \ref{tab:relative_error_option} show that all three estimation methods significantly outperform the prior, confirming that assimilating the information in the constraints is critical. Notably, a Wasserstein-based method (either the standard or smooth variant) achieves the lowest relative error in five of the six test cases. This suggests that the Wasserstein distance, as minimized by our framework, is an effective metric for regularizing the option pricing calibration problem.
\begin{table}[h!]
\begin{tabular}{|l|llll|}
\hline
 Options                & Prior   & Wasserstein & KL Divergence & Smooth Wasserstein \\
                 \hline
Up-and-out       & -0.8831 & -0.0075     & -0.0670       & -0.0937            \\
                 & -0.7302 & 0.0300      & -0.0288       & 0.0258             \\
                 \hline
Cash-or-nothing  & -0.4056 & -0.0059     & -0.0654       & -0.0695            \\
                 & -0.4641 & -0.1038     & -0.0734       & -0.0161            \\
                 \hline
Asset-or-nothing & -0.7799 & -0.0053     & -0.0141       & 0.0973             \\
                 & -0.7257 & 0.0669      & -0.0212       & 0.0227    \\
\hline
\end{tabular}
\caption{Relative error rates from the pricing of several exotic options using the estimated asset pricing measure with different frameworks: the prior measure, the proposed framework and the KL divergence framework. }
\label{tab:relative_error_option}
\end{table}

\section{Conclusions}
\label{sec:conclusion}
In this paper, we develop a modified optimal transport framework that incorporates equality constraints to meet specific requirements. We provide both explicit analytical solutions and numerical implementations to illustrate the differences between our approach and the classical KL divergence method. Additionally, we design algorithms tailored for sample-based prior distributions.

A promising direction for future research is to consider scenarios where neither the source measure nor the target measure is explicitly specified, but only subject to certain constraints. Such situations arise frequently in operations research, where supplies and demands are only constrained rather than fully defined.
\bibliographystyle{plain}
\bibliography{sample}

\subsection{The OT solution for the indicator function}
\label{appd:prop_indicator_function}
We provide the solution for \eqref{opt:constrained_ot_monge} when $K=1$ and $f$ is an indicator function \eqref{func:support} in the following proposition:
\begin{prop}
    Let $X=Y=\mR$ and $c(x,y)= \frac{1}{2}\|x-y\|^2_2$ be the pairwise cost for the objective function of the constrained optimization problem in \eqref{opt:constrained_ot_monge}, with $K=1$.
    Let $f_1$ be the indicator function in \eqref{func:support} with $\bar{f}_1=0$. Assume that $\rho\in\mc{B}(\mR)$ is the prior measure. Then the optimal target measure $\mu^*\in\mc{B}(\mR)$ is given as follows:
    \begin{equation}
        \mu^* = \rho|_{[a,b]} + c_a\delta(a) + c_b\delta(b),
    \end{equation}
    where $\rho|_{[a,b]}$ refers to the restriction of the measure $\rho$ within the interval $[a,b]$, that is, for all $A\subset \mR$,
    \begin{equation}
        \rho|_{[a,b]}(A) = \rho(A\cap [a,b]),
    \end{equation}
    $\delta(x_0)$ is the atomic measure centered at $x_0$ and 
    \begin{equation}
        c_a = \int_{-\infty}^a{p_\rho(z)dz},\;\;c_b = \int_{b}^{\infty}{p_\rho(z)dz}.
    \end{equation}
    \label{prop:indicator_single_constraint}
\end{prop}

\begin{proof}
The indicator function constraint with $f_1$ in \eqref{func:support} and $\bar{f}=0$ requires that no mass is allowed in the target measure outside the interval $[a,b]$. Thus the least costly way to bring all mass to an interval is to leave the mass that is already within the interval in place, and bring all mass outside to the closest point in the interval, i.e. the interval?s closest endpoint, which leads to the result \eqref{opt:sol}.
\end{proof}

\subsection{Proof of proposition \ref{prop: relu_single_constraint}}
\label{appd: relu_1}
\begin{prop}             
     Let \eqref{opt:constrained_ot_monge} be the optimization problem with the function $f_k$ expressed in \eqref{func:relu} and $\bar{f}_k\geq 0$ (otherwise there is no optimal measure satisfying the equality constraint) and $k=1$. Let $\rho\in\mc{B}(\mR)$ be the {atomless} prior measure with $\supp \rho = X=\mR$. 
     Then we can derive the optimal target measure $\mu^*\in\mc{B}(\mR)$ as follows: for any $A\subset (-\infty,\infty) $,
     \begin{equation}
         \mu^*(A) = \rho(T^{-1}_*(A)),
     \end{equation}
     where $T_*:X\rightarrow Y,\;T_*(x) = x + \tau(x),\;$ and $\tau: X\rightarrow \mR$ is determined by parameters $x_*$ and $\lambda$ as follows:
     \begin{equation}
       \tau(x) = \begin{cases}
           \lam & x\in [x_*,\infty), \\
           0 & \mbox{otherwise}.
       \end{cases}  
     \end{equation}
     The parameters $\lam,x_*$ are determined by solving the following two candidate systems and select the solution yielding the lower value of the objective function.

     1) 
 \begin{equation}
 \begin{cases}
x_*=\om-\lambda, \\
      \int_{\om-\lam}^{\infty}{(x-\om+\lambda)p_{\rho}(x)dx} = \bar{f};
 \end{cases}
\end{equation}

2)
\begin{equation}
    \begin{cases}
      x_* =\om - \frac{\lam}{2}, \\
      \int_{x_*}^{\infty}{(x-\om+\lambda)p_\rho(x)dx} = \bar{f}, \\
      \int_{\om}^{\infty}{(x-\om)p_\rho(x)dx} < \bar{f}.
    \end{cases}
\end{equation}

\label{prop: relu_single_constraint}
 \end{prop}

\begin{proof}

  From lemma \ref{lemma:monotone_map} we have the following claim. 

\begin{claim}
\label{claim:monotone_map}
    Suppose there exist $x_0,x_1\in X$ such that $T(x_0) \geq \om,\; T(x_1)\leq \om$. Then we can find $x_*\in X$ such that 
  \begin{equation}
  \begin{cases}
       x + \tau(x)\leq \om & x\leq x_*, \\
        x + \tau(x)> \om & x>x_*.
        \label{assume:monotone}
  \end{cases}
  \end{equation}
\end{claim} 
\begin{proof}[Proof of claim 1]
  This is a direct result of lemma \ref{lemma:monotone_map}. 
\end{proof}
The previous claim helps build the following claim:
\begin{claim}
Let $f_1$ be the `RELU' function as in \eqref{func:relu_1}. Then the equality constraint in \eqref{opt:constrained_ot_monge} can be rewritten in terms of the source distribution as follows:
\begin{equation}
    \int_{x_*}^{\infty}{(x+\tau(x)-\om)p_\rho(x)dx} = \bar{f}.
    \label{eq:target_prior_transform}
\end{equation}
\end{claim}

\begin{proof}[Proof of claim 2]
    This is an application of the rule of lazy statistician. Using the notation $ y = T(x) = x+\tau(x)$ as well as the theorem $\mE_{X\sim f_X(x)}[g(X)] = \int {g(x)f_{X}(x)dx}$, where $g(x) = (x + \tau(x) -\om)_+$.  

    By lemma \ref{lemma:monotone_map} we know $ (x+\tau(x)-\om)\mathbf{1}_{x\geq x_*} = (y-\om)\mathbf{1}_{y\geq \om}=:g\circ T^{-1}(y)$. Thus $E_{X}[g(X)] = E_{Y}[g\circ T^{-1}(Y)]$, the right hand side of which corresponds to the RHS of the equation \eqref{eq:target_prior_transform}.
\end{proof}
The previous claim substitutes the the original equality constraint in the optimization problem \eqref{opt:constrained_ot_monge} with \eqref{eq:target_prior_transform} equivalently. 

Regarding the shifting on $[-\infty,x_*]$, we have 
\begin{claim}
    For $x\leq x_*$ the optimal shift 
function
\begin{equation}
    \tau^*(x) = 0.
    \label{sol:relu_no_shift}
\end{equation}
\end{claim}
\begin{proof}[Proof of claim 3]
    By claim \ref{lemma:monotone_map}  for $x\leq x_*$, we have $x+\tau(x)\leq \om$  so $(x+\tau(x)-\om)_+ = 0$. By the expression of $f$ in \eqref{func:relu}, when $x\leq x_*$ the term $x+\tau(x)$ contributes nothing to the expectation constraint \eqref{opt:constrained_ot_monge} for $K=1$ no matter how large $\tau(x)$ is, yet $|\tau(x)|^2$ is always nonnegative.   
\end{proof}

When $x>x_*$ (or $x + \tau(x)>\om$),  we denote $\iota(x) = \tau(x) - (\om - x)$. Thus $\iota(x)\geq 0$. Then the optimization problem is converted into 
\begin{equation}
\begin{aligned}
     \underset{\substack{ \iota(x),x_* \\ \iota(x)\geq 0}}{\min}&\frac{1}{2} \int_{x_*}^{\infty}{|\iota(x) + \om - x|^2 p_\rho(x)dx}, \\
     \text{s.t.}&\;\int_{x_*}^{\infty}{\iota(x)p_\rho(x)dx} = \bar{f}.
     \label{opt:x_geq_x_star}
\end{aligned}
\end{equation}

 We associate a Lagrange multiplier $u:X\rightarrow \mR,\;u\geq 0$ to the inequality $\iota\geq 0$.  The (partial) corresponding Lagrangian $\mc{L}$ is 
\begin{equation}
    \mc{L}(\iota,\lambda,u) = \int_{x_*}^{\infty}{[\frac{1}{2}|\iota(x) + \om - x|^2 - (u(x)+\lambda)\iota(x)]  p_\rho(x)dx}+ \lambda \bar{f}. 
\end{equation}
So taking the first-order condition with respect to $\iota$ (variational derivative), we get 
\begin{equation}
    0\equiv \frac{\delta \mc{L}}{\delta \iota} = (\iota(x) + \om - x) - u(x) - \lambda.
    \label{eq:first_order_condition}
\end{equation}

Considering also the complementary slackness condition 
\begin{equation}
u(x)\iota(x) = 0,\;\forall x\geq x_*.
\label{eq:complentary_slackness}
\end{equation}    
From \eqref{eq:complentary_slackness} we know for every $x\geq x_*$, either $u(x) = 0$ or $\iota(x) = 0$ or both. We convert the optimization problem over all $\tau$ (or equivalently, $\iota$) into the one over all $x_*\in\mR$ and $\tau(x)\mathbf{1}_{x\geq x_*}$. Here we discuss by cases.

\paragraph{Case 1: $u(x)=0$ } Then $\iota(x)>0$ strictly (from the monotone assumption, we already know $x+\tau(x)\geq \om$ and now the complementary slackness condition implies under this condition $\iota(x)\neq 0$, but according to \eqref{eq:first_order_condition}the sum is zero, so $\iota$ must be positive). And we further know $\iota(x)= x-\om +\lam$, which, together with the strict positivity of $\iota$, implies $x>\om-\lam$.

\paragraph{Case 2: $\iota(x)=0$. }Then $u(x)>0$ strictly,  and $u(x)=\om-x-\lambda$, which implies $x<\om-\lambda$.

In short, we can summarize the expression of $\iota_*$ when $x\geq x_*$ as  
\begin{equation}
 \iota_*(x) =    \begin{cases}
        x-\om + \lambda & x>\om -\lam \\
        0 & \mbox{o/w}.
    \end{cases}
\end{equation}

Regarding the relationship between $x_*$ and $\om-\lam$ there are two choices: either $x_*>\om-\lam$ strictly or $x_*\leq \om-\lambda$.

a)  Suppose the following holds true: 
\begin{equation}
    x_*> \om-\lambda(x_*).
    \label{ineq:x_lambda}
\end{equation}
We know that for all $x>x^*$ condition $x>\om-\lam$ is satisfied, which leads to the expression of $\iota$ as in case 1. 
 The complementary slackness conditions still apply.
Notice $\lambda$ is dependent on $x_*$ via the following equality constraint:
\begin{equation}
   H(x_*,\lambda):= \int_{x_*}^{\infty}{(x-\om+\lambda)p_\rho(x)dx} = \bar{f}.
    \label{eq:lambda2}
\end{equation}
With a little abuse of notation we here use $\lambda(x_*)$ to stress such dependency. 
Accordingly, 
\begin{equation}
    \tau^*(x) = \begin{cases}
        \lambda(x_*)& x>x_*,  \\
        0 & \mbox{o/w}.
    \end{cases}
    \label{sol:case_2_tau}
\end{equation}
 Applying first-order condition on it gives:
\begin{equation}
   x_*  = \om - \frac{1}{2}\lam(x_*).
     \label{eq:find_x_star}
\end{equation}

Equations \eqref{sol:case_2_tau},\eqref{eq:find_x_star} and \eqref{eq:lambda2} provide a candidate solution for $\tau^*$.
 We end up with the target distribution, which is split into two intervals $[x_*+\lam(x_*),\infty)$ and $(-\infty,x_*]$. 

In the meantime, to assure that $x_*$ in \eqref{eq:find_x_star} satisfies the condition \eqref{ineq:x_lambda}, there shall exist such $\lam>0$ that 
\begin{equation}
    Q_0(\lam):=\int_{\om-\frac{\lam}{2}}^{\infty}{(x-\om+\lam)p_\rho(x)dx} = \bar{f}.
\end{equation}
Since $\rho$ is atomless, $p_\rho$ is continuous and so is $Q_0$. Considering $Q_0(\lam)\rightarrow \infty$ as $\lam\rightarrow \infty$, the existence of such a $\lam>0$ is equivalent to $Q_0(0) := \int_{\om}^{\infty}{(x-\om)p_\rho(x)dx} < \bar{f}$.
 \color{black}

b) If on the other hand $$ x_*\leq \om-\lam,$$ we know for all $x>x_*$, either $x_*<x<\om-\lam$ or $x>\om-\lam$. Thus by discussion of complementary slackness we know 
\begin{equation}
    \iota(x)=
    \begin{cases}
        x - \om + \lam,  & x>\om-\lam, \\
        0, & x_*<x<\om-\lambda, 
    \end{cases}
    \label{sol:iota}
\end{equation}
Substituting the expression of $\iota(x)$ in \eqref{sol:iota} into the equality constraint, we get
\begin{equation}
\begin{aligned}
    H^-(\lam)&:= \int_{\om-\lam}^{\infty}{(x-\om+\lambda)p_{\rho}(x)dx} =: \bar{f}, 
\end{aligned}  
    \label{eq:lambda}
\end{equation}
which only concerns $\lam$. Since $\rho$ is atomless and the left hand side of \eqref{eq:lambda} is continuous and increasing in terms of $\lambda$, Thus the solution to \eqref{eq:lambda} exists, which is denoted as $\lambda_*$. Now we aim to find out $x_*$. From the monotone assumption \eqref{assume:monotone} we know an inequality constraint from the the expression of $\tau^*$ in \eqref{eq:tau}:

\begin{equation}
    \tau^*(x) = 
    \begin{cases}
        \lambda_*,  & x>\om - \lambda_*,\\
        \om - x, &  x_*<x<\om-\lambda_*,\\
        0, & x<x_*,
    \end{cases} 
    \label{eq:tau}
\end{equation}
where $\lam_*$ can be determined through \eqref{eq:lambda}. 
Using first-order condition we get 
\begin{equation}
    x_* = \om - \lam_*.
    \label{x_star_result}
\end{equation}

Summarizing the results of \eqref{x_star_result} and \eqref{eq:tau}, we can now know the optimal shifting function $\tau^*$ as. 
\begin{equation}
    \tau^*(x) = 
    \begin{cases}
        \lambda_*  & x\geq \om - \lambda_*=x_*,\\
        0 & x<\om-\lam_*.
    \end{cases} 
    \label{eq:tau_final_solution_case1}
\end{equation}

It is clear to see that for both situations a) and b) we can always find proper $\lam$ and $x_*$. So the best strategy is to compute the values of the objective function solved by both ways and select the way leading to a lower value. 
\end{proof}

\paragraph{Remark:} If $\supp p_\rho = (a,b]$ instead of $\mR$, the two candidates for $x_*,\lam$ still hold, provided that $x_*\in (a,b]$. In addition, there are two more candidate systems when $x_*=a, x_*=b$ and the corresponding $\lam$ are computed using \eqref{eq:lambda2}.

\paragraph{Remark:} The proof above carries over to the situation where $\rho$ is a discrete measure by replacing integrals with finite sums. Specifically, if $\rho$ is supported on finite points: $\{x_{(i)}\}_i$, that is,
\begin{equation}
    \rho = \sum_{i}{c_i\delta(x_{(i)})}.
\end{equation}
Then the candidate systems are 
 \begin{align}
\text{1)}\;x_*&=\om-\lambda, \;\sum_{i:x_{(i)}>x_*}{c_i(x_{(i)}-\om+\lambda)} = \bar{f}; \\
\text{2)}\;x_* &=\om - \frac{\lam}{2}, \;
     \sum_{i:x_{(i)}>x_*}{c_i(x_{(i)}-\om+\lambda)} = \bar{f}.
 \end{align}

 \subsection{Proof of proposition 3}
\label{appd:relu_K}
 \begin{prop}                     
     Let \eqref{opt:constrained_ot_monge} be the optimization with the function $f_k$ being the class of `RELU' functions in \eqref{func:relu} and $\bar{f}_k\geq 0$ (otherwise there is no optimal measure satisfying the equality constraint).
     Then we can derive the optimal target measure $\mu^*$ as follows: for any $A\subset (-\infty,\infty) $,
     \begin{equation}
         \mu^*(A) = \rho(A-\tau_0),
     \end{equation}
     where 
     \begin{equation}
     \tau_0 = \begin{cases}
         0 & \mbox{$x<x_{1*}$}, \\
         \tau_{k*} & x\in [x_{k*},x_{k+1*}],\;\; k=1,2,\dots,K-1\\
         \tau_{K*} & x>x_{K*}.
     \end{cases}
 \end{equation}
where $\tau_{k*},x_{k*}$ are to be specified in the following recursive way:

1) When $k=K$, we have 
\begin{equation}
    \tau_K = \begin{cases}
        \lambda_{K*}  & x\in[x_{K*}+\lam_{K*},\infty], \\
        0 & x\in[x_{K*},x_{K*}+\lam_{K*}],
    \end{cases}
\end{equation}
where $\lam_{K*}, x_{K*}$ are solutions of one of the following equations:

1-1) 
 \begin{equation}
 \begin{cases}
x_{K*}=\om_K-\lam_{K*}, \\
      \int_{\om_K-\lam_{K*}}^{\infty}{(z-\om_K+\lam_{K*})p_{\rho}(z)dz} = \bar{f}_K;
 \end{cases}
\end{equation}

1-2)
\begin{equation}
    \begin{cases}
      x_{K*} =\om_K - \frac{\lam_{K*}}{2}, \\
      \int_{x_{K*}}^{\infty}{(z-\om_K+\lam_{K*})p_\rho(z)dz} = \bar{f}_K.
    \end{cases}
\end{equation}

2) 
\begin{equation}
    \tau^*_k(x) = 
    \begin{cases}
        \om_{k+1} - x,  & x\in [\om_{k+1} - \lam_{k*},x_{k+1*}],\\
        \lambda_{k*}, &  x \in [\om_k-\lam_{k*} ,\om_{k+1}-\lam_{k*}],\\
        \om_k-x & x\in [x_{k*},\om_k-\lam_{k*}],
    \end{cases}  \label{eq:tau_final_solution_case1}
\end{equation}
For each $k\in [K-1]$. For simplicity, we define
\begin{equation}
\begin{aligned}
\Delta \om_k & = \om_{k+1}-\om_k;\\
    H^+_k(x,\lam) &= \int_{x}^{x_{k+1*}}{(z-\om_k+\lam)p_\rho(z)dz};\\
    H^{-}_k(x,\lam) & =\int_{x}^{\om_{k+1}-\lam}{(z-\om_k+\lam)p_\rho(z)dz} + \Delta \om_k\int_{\om_{k+1}-\lam}^{x_{k+1*}}{p_\rho(z)dz};\\
    \Delta \tilde{f}_k(\lam)& = \bar{f}_k- \bar{f}_{k+1} + \Delta \om_k\int_{x_{k+1*}-\lam}^{\infty}{p_\rho(z)dz}.
\end{aligned}
\end{equation}
Then $\lambda_{k*}$ and $x_{k*}$ satisfying one set of the following system equations: 
\begin{equation}
\begin{aligned}
 \text{2-1)} \qquad &x_{k*} = \om_k -\frac{1}{2}\lam_{k*},\;H^+_k(x_k,\lam_k) = \Delta\tilde{f}_k; \\
 \text{2-2)} \qquad    &x_{k*} = \om_k -\frac{1}{2}\lam_{k*},\;\;H^{-}_k(x_k,\lam_k)= \Delta\tilde{f}_k;\\
 \text{2-3)} \qquad  & x_{k*} = \om_k -\lam_{k*},\;\;H^+_k(x_k,\lam_k) = \Delta\tilde{f}_k; \\
 \text{2-4)} \qquad    &x_{k*} = \om_k -\lam_{k*},\;\;H^{-}_k(x_k,\lam_k)= \Delta\tilde{f}_k; \\
 \text{2-5)} \qquad& \lam_k\leq \om_{k+1}-x_{k+1}, \;\Delta\om_k\int_{x_{k*}}^{\infty}{p_\rho(z)dz} = \Delta f_k,\; .
\end{aligned}
\end{equation}
\label{prop:relu_K}
 \end{prop} 
  
  \begin{proof}    
  Similar to the proof of proposition 2,
we denote $ y = x + \tau(x)$. As an extension of claim 1 in the previous proof, we could find $x_{1*}\leq x_{2*}\leq \dots \leq x_{K*}$ such that 
\begin{equation}
    \begin{cases}
      x+\tau(x)\leq \om_k & x<x_{k*}, \\
      x+\tau(x)\geq \om_k & x\geq x_{k*}, k =1,2,\dots, K. 
      \label{ineq: turning_points_K}
        \end{cases}
\end{equation}
Inspired by the $K$ inequalities \eqref{ineq: turning_points_K}, we can decompose the real line into $K$ intervals: $(\infty,x_{1*}],(x_{1*},x_{2*}],\dots,(x_{K*},\infty)$, where $x_{1*}\leq x_{2*}\leq,\dots$ are to be computed later. We can rewrite the $K$ equality constraints by decompose the integrals on the LHS of the constraints into the ones on intervals:

\begin{align}
     \underset{\tau}{\min}& \int_{-\infty}^{x_1}{|\tau(x)|^2f_\rho(x)dx} + \sum_{k=1}^{K-1}{\int_{x_{k*}}^{x_{k+1*}}{|\tau(x)|^2 d\rho}}+\int_{x_{K*}}^{\infty}{|\tau(x)|^2 d\rho} \\
     \text{s.t.}&\;\int_{x_{K*}}^{\infty}{\tau(x)d\rho} = \bar{f}_K + (1-c_{K*})\om_K - s_{K*}, \\
     & \int_{x_{k*}}^{x_{k+1*}}{\tau(x)d\rho} = \bar{f}_k - \bar{f}_{k+1} + (1-c_{k*})\om_k- (1-c_{k+1*})\om_{k+1*}  \\ 
     &\qquad\qquad\qquad\;\; - (s_{k*}-s_{k+1*}),\;\; k\in [K-1],
\label{opt:K_constraints_separation}
\end{align}
where we denote 
\begin{equation}
    c_{k*} = \int_{-\infty}^{x_{k*}}{p_{\rho}(x)dx},\;\;
    s_{k*}  = \int_{x_{k*}}^{\infty}{xp_{\rho}(x)dx}.
\end{equation}

We can solve the general case in a reverse dynamic programming way. That is, we begin with $k=K$, solving $\tau_K$ and $x_{K*}$, and then go backwards until $k=1$.

\paragraph{Step 1:} When $k=K$, we aim to solve the following problem:
\begin{equation}
    \begin{aligned}
     \underset{\substack{ \iota_K(x),x_{K*} \\ \iota_K(x)\geq 0}}{\min}& \int_{x_{K*}}^{\infty}{|\iota_K(x) + \om_K - x|^2 d\rho},   \\
     \text{s.t.}&\;\int_{x_{K*}}^{\infty}{\iota_K(x)d\rho} = \bar{f}_K.
     \label{opt:multi_constraint_K_last}
    \end{aligned}
\end{equation}
The problem \eqref{opt:multi_constraint_K_last} can be solved by the same way as discussed in the previous paragraph to find $\iota_K(x), x_{K*}$. Denote $\tau_{K*}(x) = \iota_{K*}(x) + \om_K - x$. We have 

\begin{equation}
    p_{\mu}(y) = \begin{cases}
        p_{\rho}(y-\tau_K(x_{K*})), & y\geq x_{K*}+\tau_K(x_{K*}), 
        \label{sol:k=K}
    \end{cases}
\end{equation}

\textbf{Step 2:} When $k=1,2,\dots,K-1$, we aim to find the optimal $\tau^*$ when $x\in [x_{k*},x_{k+1*}]$. 

Having solved the case $k=K$, we substitute ${x_{K*}}$ into \eqref{opt:multi_constraint_K_k_k1} and obtain the optimal solution $\tau_{K-1*}$ parameterized by $x_{K-1*},\lambda_{K-1}$, which could then be used in \eqref{opt:multi_constraint_K_k_k1} to solve $\tau_{K-2*}$, etc.  Let $\iota_k(x) = \tau(x) - (\om_k - x)$ for $k=1,2,\dots, K$ represent the shifting $\iota$ restricted within the interval $[x_{k*},x_{k+1*}]$. 
 From the monotone mp assumption we know $\iota_k(x)\geq 0$. Unlike the case of a single constraint, we also know that $\tau(x) - (\om_{k+1}-x)\leq 0$, implying $\iota_k(x)\leq \om_{k+1}-\om_k$ within the interval $[x_{k*},x_{k+1*}]$. 
 
 Inspired by the lemma, we know the optimal shift function $\tau^*$ is a constant in the region $[x_{k*},x_{k+1*}]$, can be denoted as $\tau^*_k$. Then the optimization problem is converted into 
\begin{equation}
\begin{aligned}
     \underset{\substack{ \iota_k(x),x_{k*} \\ 0\leq\iota_k(x) \\
     \iota_k(x) \leq \om_{k+1}-\om_k}}{\min}& \int_{x_{k*}}^{x_{k+1*}}{|\iota_k(x) + \om_k - x|^2 d\rho}, \\
     \text{s.t.}&\;\int_{x_{k*}}^{x_{k+1*}}{\iota_k(x)d\rho} = \bar{f}_k-\bar{f}_{k+1} - \int_{x_{k+1*}}^{\infty}{(\om_{k+1}-\om_k)p_\rho(x)dx} 
    ,\;\\
    &k=1,2,\dots, K-1 .
\label{opt:multi_constraint_K_k_k1}
\end{aligned}
\end{equation}
We introduce another Lagrange multiplier $v_k\geq 0$ to pair with the extra inequality constraint $\iota_k(x)\leq \om_{k+1}-\om_k =:\Delta\om_k$. Also for brevity we denote 
\begin{equation}
\Delta\tilde{f}_k = \bar{f}_k-\bar{f}_{k+1} - \int_{x_{k+1*}}^{\infty}{(\om_{k+1}-\om_k)p_\rho(x)dx} 
\end{equation}

We denote $u_k(x)$ as the Lagrange multiplier for the inequality constraint $\iota_k(x)\geq 0$. The KKT condition as follows: 

\begin{align}
   \iota_k(x) + \om_k - x - u_k(x) + v_k(x) - \lam_k&\equiv 0,
\label{eq:first_order_condition_k} \\
u_k(x)\iota_k(x) &= 0,\\
v_k(x)(\iota_k(x)-\Delta\om_k) &= 0.
\end{align}
1)If $\iota_k(x)< \Delta \om_k$ strictly, the last condition is active and thus $v_k(x)=0$. In this situation the previous discussions on the case $K=1$ an carry over here:

1a) If $\iota_k(x)>0$ strictly, $u_k(x)=0$ and we get $\iota_k(x) = x-\om_k + \lam_k$. The strict positivity of $\iota_k$ implies this could happen when $x>\om_k - \lam_k$. Considering $\iota_k<\Delta\om_k$ strictly, we revise the region into $\om_k-\lam_k<x<\om_{k+1}-\lam_k$.

1b) If $\iota_k(x)=0$, we know $u_k(x)>0$ strictly. Together we have known $v_k(x)=0$. Substituting choices of $\iota_k,u_k,v_k$ into the first-order equation, we know $x = \om_{k}-\lam_k-u_k(x)$, implying this happens when $x<\om_k-\lam_k$.

2)If $\iota_k(x)=\Delta\om_k$ for some $x\in [x_{k*},x_{k+1*}]$, we know $v_k(x)>0$ stricly and $u_k(x)=0$ (since $\iota_k(x)>0$ strictly as well). As a result, we conclude $v_k(x) = \lam_k + x -\om_{k+1}>0$ and thus $x>\om_{k+1}-\lam_k$. In this situation $\tau(x) = \om_{k+1}-x$.

To summarize, we can derive $\iota_k(x)$ as follows:
\begin{equation}
  \iota_k(x) = \begin{cases}
    0, & x<\om_k -\lam_k, \\
      x - \om_k + \lam_k, & \om_k -\lam_k<x<\om_{k+1}-\lam_k, \\
    \Delta\om_k, & x>\om_{k+1}-\lam_k. 
  \end{cases} 
  \label{sol:optimal_iota_k}
\end{equation}

The previous expression \eqref{sol:optimal_iota_k} specifies the optimal shift $\tau^*_k(x)$ within $[x_{k*},x_{k+1*}]$ (or equivalently, $\iota_k(x)$) when $x$ belongs to different sub-regions. However, the relationship among $x_{k*},x_{k+1*},\om_{k}-\lam_k,\om_{k+1}-\lam_k$ remains unknown (except we assume $x_{k*}\leq x_{k+1*}$ and we know $\om_k\leq \om_{k+1}$). There are in total six cases below: For different cases, we could rewrite the objective function $J_k = J|_[x_{k*},x_{k+1*}]$ into at most three terms.

0) If $x_{k*} \leq x_{k+1*}\leq \om_k - \lam_{k*}{\leq \om_{k+1}-\lam_k }$, objective function only contains one term:
\begin{equation}
J_k(x_{k*}) = \int_{x_{k*}}^{x_{k+1*}}{|\om_k-x|^2d\rho}
\end{equation}

The equality constraint reduces to 
\begin{equation}
\int_{x_{k}}^{x_{k+1*}}{0 d\rho}
    = \bar{f}_k-\bar{f}_{k+1} - \int_{x_{k+1*}}^{\infty}{(\om_{k+1}-\om_k)p_\rho(x)dx}.
\end{equation}
Which is impossible unless the RHS is zero.

1) If $x_{k*}<\om_{k} - \lam_k {<x_{k+1*}}< \om_{k+1} - \lam_k$,  the objective function writes as follows:
\begin{equation}
      J_k(x_{k*}) = \int_{x_{k*}}^{\om_{k}-\lam_k}{|\lam_k|^2p_\rho(x)dx}  +\int_{\om_{k}-\lam_k}^{x_{k+1*}}{|\om_{k+1}-x|^2 p_\rho(x)dx}.
\end{equation}
with the equality constraint:
\begin{equation}
\begin{aligned}
H_k(\lambda_k):&= \int_{\om_k-\lam_{k}}^{x_{k+1}}{(x-\om_k+\lambda_k)d\rho}= \Delta \tilde{f}_k.
\end{aligned}
\label{eq:lambda_K_case_two}
\end{equation}
The optimal $\lambda_k$ can be specified via solving \eqref{eq:lambda_K_case_two}. The optimal $x_k = \om_k-\lam_k$.

2) If $x_{k*}\leq \om_k - \lam_{k*}{\leq \om_{k+1}-\lam_k \leq x_{k+1*} }$ :
\begin{equation}
    J_k(x_{k*}) = \int_{x_{k*}}^{\om_k-\lam_{k*}}{|\om_k-x|^2d\rho} + \int_{\om_k-\lam_{k*}}^{\om_{k+1} - \lam_{k*}}{|\lam_{k*}|^2d\rho}  +\int_{\om_{k+1}-\lam_{k*}}^{x_{k+1*}}{|\om_{k+1}-x|^2 d\rho}
\end{equation}
where we have according to \eqref{sol:optimal_iota_k} and recall $\tau_k(x) = \iota_k(x) + \om_k - x$:
\begin{equation}
    \tau^*_k(x) = 
    \begin{cases}
        \om_{k+1} - x,  & x\in [\om_{k+1} - \lam_{k*},x_{k+1*}],\\
        \lambda_{k*}, &  x \in [\om_k-\lam_{k*} ,\om_{k+1}-\lam_{k*}],\\
        \om_k-x & x\in [x_{k*},\om_k-\lam_{k*}],
    \end{cases}  \label{eq:tau_final_solution_case1}
\end{equation}
and $\lam_{k*}$ can be computed via the following equation 
\begin{equation}
\begin{aligned}
H_k(\lambda_k):&= \int_{\om_k-\lam_{k}}^{{\om_{k+1}-\lam_k}}{(x-\om_k+\lambda_k)d\rho}{+\int_{\om_{k+1}-\lam_k}^{x_{k+1*}}{\Delta\om_k d\rho}} \\
    &= \Delta \tilde{f}_k.
\end{aligned}
    \label{eq:lambda_K_case_one}
\end{equation}
We notice $\frac{\partial H_k}{\partial \lam_k} = \int_{\om_k-\lam_k}^{\om_{k+1}-\lam_k}{p_\rho(x)dx}$.
Similar to the single-constraint case, the constraint \eqref{eq:lambda_K_case_one} only concerns $\lam_k$ and is independent of $x_k$. Thus $\lam_k$ is computed via \eqref{eq:lambda_K_case_one}. So $x_{k*} = \om_k-\lam_{k*}$. After computing the optimal shift $\tau_k$ for all $k\in [K]$, we can now write out the optimal target distribution $p_{\mu}$ as follows:
\begin{equation}
    p_\mu(y) = \begin{cases}
         p_\rho(y-\lam_{k*}) & y\in [x_{k*}+\lam_{k*},\om_{k+1}],\\
         c_k\delta_{\om_{k+1}}(y)& y\in[\om_{k+1},x_{k+1*}+\lam_{k*}],
         \label{sol:case1_f_mu}
    \end{cases}
\end{equation}
with $c_k = \int_{\om_{k+1}-\lam_k}^{x_{k+1*}}{p_\rho(x)dx}$.

3) If $\om_{k} - \lam_k <x_{k*}< \om_{k+1} - \lam_k{<x_{k+1*}}$, the objective function contains two terms: 
\begin{equation}
    J_k(x_{k*}) = \int_{x_{k*}}^{\om_{k+1}-\lam_k}{|\lam_k|^2p_\rho(x)dx}  +\int_{\om_{k+1}-\lam_k}^{x_{k+1*}}{|\om_{k+1}-x|^2 p_\rho(x)dx}.
\end{equation}
The complementary slackness conditions still apply. 
Similarly $\lambda_k$ is dependent on $x_{k*}$ via the following equality constraint:
\begin{equation}
\begin{aligned}  H_k(x_{k*},\lambda_k):&= \int_{x_{k*}}^{{\om_{k+1}-\lam_k}}{(x-\om_k+\lambda_k)d\rho} +\int_{{\om_{k+1}-\lam_k}}^{{x_{k+1*}}}{\Delta \om_k d\rho}  \\
     &= \bar{f}_k-\bar{f}_{k+1} - \int_{x_{k+1*}}^{\infty}{(\om_{k+1}-\om_k)p_\rho(x)dx}.
\end{aligned}
    \label{eq:lambda_K_case}
\end{equation}
We can apply implicit function theorem to compute its derivative as:

\begin{equation}
    \lambda'_k(x_k) = -\frac{\partial H_k/\partial x_k}{\partial H_k/\partial \lambda_k} = -\frac{-(x_k-\om_k + \lam_k)p_\rho(x_k)}{\int_{x_k}^{\om_{k+1}-\lam_k}{p_\rho(x)dx}} = \frac{(x_k-\om_k + \lam_k)p_\rho(x_k)}{\int_{x_k}^{\om_{k+1}-\lam_k}{p_\rho(x)dx}}
\end{equation}
For simplicity, we denote $I_k = \int_{x_k}^{\om_{k+1}-\lam_k}{p_\rho(x)dx}$.
\color{black}

With a little abuse of notation, we here use $\lam_k(x_*)$ to stress such dependency. 
Accordingly, 
\begin{equation}
    \tau^*_k(x) =
        \begin{cases}    \lambda_k(x_{k*})\qquad &x\in [x_{k*},\om_{k+1}-\lam_k], \\ 
        \om_{k+1}-x,& x\in[\om_{k+1}-\lam_k,x_{k+1*}].
    \end{cases}   \label{eq:tau_final_solution_case2}
 \end{equation} 
To find out the optimal $x_{k*}$, we apply first order condition:
\begin{equation}
\begin{aligned}
     0 &= \frac{dJ(x_k)}{dx_k}\\
     &= 2\lam_k(x_k)\lam'_k(x_k)\Big[\int_{x_k}^{\om_{k+1}-\lam_{k}(x_k)}{p_\rho(x)dx}-\frac{1}{2}\lam_k(x_k)p_\rho(\om_{k+1}-\lam_k(x_k))\Big]-\lam^2_k(x_k)p_\rho(x_k)  \\
     & +|\lambda_k(x_k)|^2\lambda'(x_k)p_\rho(\om_{k+1}-\lam_k(x_k)) \\
     & = \lam_k(x_k)\Big[2\lam'_k(x_k)\int_{x_k}^{\om_{k+1}-\lam_{k}(x_k)}{p_\rho(x)dx}-\lam_k(x_k)p_\rho(x_k)\Big] \\
     & = \lam_k(x_k)p_\rho(x_k)\Big( \frac{2(x_k-\om_k+
   \lam_k)}{I_k} \int_{x_k}^{\om_{k+1}-\lam_{k}(x_k)}{p_\rho(x)dx}-\lam_k(x_k)\Big), \\
\end{aligned}
\end{equation}

 Denote
 \begin{equation}
 \begin{aligned}
   G_k(x_k) & =2\frac{x_k-\om_k+\lam_k}{I_k}\Big(\int_{x_k}^{\om_{k+1}-\lam_{k}(x_k)}{p_\rho(x)dx}\Big)-
\lam_k(x_k) \\
&=2(x_k-\om_k) +\lam_k.
 \end{aligned}   
 \end{equation}
So the zeros of $G_k$ are $x_k = \om_k-\frac{1}{2}\lam_k$. 

  where 
  \begin{equation}
      \begin{aligned}
          A_k &= x_k - \om_k +\lambda_k(x_k), \\
          D_k&=\int_{x_k}^{x_{k+1*}}{p_\rho(x)dx}, \\
          I_k&=\int_{x_k}^{\om_{k+1}-\lam_k(x_k)}{p_\rho(x)dx}.
      \end{aligned}
  \end{equation}
The first-order condition leads to $G_k(x_k) = 0$ or $\lambda_k(x_k)=0$.
 
The solutions \eqref{eq:tau_final_solution_case1} and \eqref{eq:tau_final_solution_case2} characterize the solutions for general $k$ when $k=1,2,\dots, K-1$.

In this case, we have
\begin{equation}
    p_\mu(y) =
    \begin{cases}
         p_\rho(y-\lam_{k*}),&\;y\in[\lam_{k*}+x_{k*},\om_{k+1}] \\
         c_k\delta_{k+1}(y)&\;y\in[\om_{k+1},\lam_k+x_{k+1}].
    \end{cases}  
\label{sol:k<K_2}
\end{equation}
Notice here the intervial is included in $[\om_k,\om_{k+1}]$.

4) If $\om_{k} - \lam_k <x_{k*}< {x_{k+1*}<}\om_{k+1} - \lam_k$, the objective function contains two terms: 
\begin{equation}
    J_k(x_{k},\lam_k) = \int_{x_{k}}^{x_{k+1*}}{|\lam_k|^2p_\rho(x)dx}.
\end{equation}
And the equality constraint is the same as in 
\begin{equation}
\begin{aligned}
     H_k(x_{k*},\lambda_k):&= \int_{x_{k*}}^{x_{k+1*}}{(x-\om_k+\lambda_k)p_\rho(x)dx} = \Delta\tilde{f}_k.
\end{aligned}
    \label{eq:lambda_K_case_three}
\end{equation}
Then we apply first-order condition on $J$ and get
\begin{equation}
    \begin{aligned}
        0\equiv& \frac{\partial J}{\partial x_k} = 2\lam_k\lam'_k\int_{x_{k}}^{x_{k+1*}}{p_\rho(x)dx} - \lam^2_kp_\rho(x_k).
    \end{aligned}
\end{equation}
From the previous expression \eqref{eq:lambda_K_case_two} we apply implicit function theorem to get 
\begin{equation}
    \lambda'_k(x_k) = \frac{(x_k -\om_k+\lam_k)p_\rho(x_k)}{D_k}
\end{equation}
We finally get 
\begin{equation}
    \lam_k = -2(x_k-\om_k).
\end{equation}

6) Finally if $x_{k*}\geq \om_{k+1}-\lam_k$,  or equivalently  $x_{k+1*}\geq x_{k*}\geq \om_{k+1}-\lam_k> \om_{k}-\lam_k$ 
the objective contains one term:
\begin{equation}
    J_k(x_{k*}) = \int_{x_{k*}}^{x_{k+1*}}{|\om_{k+1}-x|^2 d\rho}.
\end{equation}
with $x_{k*}$ meeting the following constraint 
\begin{equation}
    \int_{x_{k*}}^{x_{k+1*}}{\Delta \om_k p_\rho(x)dx} = \bar{f}_k-\bar{f}_{k+1} - \int_{x_{k+1*}}^{\infty}{(\om_{k+1}-\om_k)p_\rho(x)dx}.
\end{equation}
or equivalently,
\begin{equation}
\Delta\om_k\int_{x_{k*}}^{\infty}{d\rho} = \bar{f}_k-\bar{f}_{k+1}
\label{eq:constraint_K_both_large}
\end{equation}
When $\Delta\om_k\geq \bar{f}_k-\bar{f}_{k+1}$, the equation \eqref{eq:constraint_K_both_large} specifies a unique $x_{k*}$, with $\lambda_k$ chosen to anything no bigger than $\om_{k+1}-x_k$ and $\tau(x) = \om_{k+1}-x$. We have 
\begin{equation}
    p_\mu(y) = c\delta_{\om_{k+1}}(y),\;y\in [\om_{k+1},\om_{k+1}+x_{k+1}-x_k]
\end{equation}
with $c= \int_{x_k}^{x_{k+1}}{p_\rho(x)dx}$.

\paragraph{Step 3:} For $x\leq x_{1*}$ the optimal shift 
function
\begin{equation}
    \tau^*(x) = 0,
\label{sol:relu_no_shift_multiple_constraint}
\end{equation}
as is similar from proposition. Thus the target distribution coincides with the source distribution in the very region:
\begin{equation}
    p_\mu(y) = p_\rho(y),\;y\in (-\infty,x_{1*}).
    \label{sol:k<1}
\end{equation}

Summarizing the expressions \eqref{sol:k<1},\eqref{sol:k<K_2}, \eqref{sol:case1_f_mu} and \eqref{sol:k=K} produces the resulting target density, with the addition of densities in overlapping regions and $0$ in noncovering regions.
\end{proof}

\end{document}